\documentclass{article} 
\usepackage{iclr2023_conference,times}

\usepackage{amsmath,amsfonts,bm}

\def\eqref#1{equation~\ref{#1}}

\def\1{\bm{1}}

\def\rmV{{\mathbf{V}}}
\def\rmW{{\mathbf{W}}}

\def\ermH{{\textnormal{H}}}

\def\ermM{{\textnormal{M}}}

\def\ermU{{\textnormal{U}}}
\def\ermV{{\textnormal{V}}}
\def\ermW{{\textnormal{W}}}

\def\ermY{{\textnormal{Y}}}

\def\vb{{\bm{b}}}

\def\vx{{\bm{x}}}

\def\vz{{\bm{z}}}

\def\evw{{w}}
\def\evx{{x}}

\DeclareMathAlphabet{\mathsfit}{\encodingdefault}{\sfdefault}{m}{sl}
\SetMathAlphabet{\mathsfit}{bold}{\encodingdefault}{\sfdefault}{bx}{n}

\newcommand{\E}{\mathbb{E}}

\newcommand{\R}{\mathbb{R}}

\newcommand{\norm}[1]{\left\lVert#1\right\rVert}

\usepackage{hyperref}
\usepackage{url}
\usepackage{graphicx}
\usepackage{subcaption}
\usepackage{multirow}
\usepackage{booktabs}       
\usepackage{wrapfig}

\usepackage{amssymb}
\usepackage{mathtools}
\usepackage{amsmath}
\usepackage{amsthm}
\newtheorem{theorem}{Theorem}[section]

\newtheorem{lemma}[theorem]{Lemma}

\newtheorem{definition}[theorem]{Definition}

\newtheorem{insight}[theorem]{Insight}

\title{Dynamical Isometry for Residual Networks}

\author{Advait Gadhikar \\
CISPA Helmholtz Center for Information Security\\
Saarbrücken 66123, Germany\\
\texttt{advait.gadhikar@cispa.de} \\
\And
Rebekka Burkholz\\
CISPA Helmholtz Center for Information Security\\
Saarbrücken 66123, Germany\\
\texttt{burkholz@cispa.de} \\
}

\usepackage{xspace}

\newcommand{\init}{\textsc{Risotto}\xspace}
\DeclareUnicodeCharacter{0229}{\textcolor{red}{BAD!!}}
\usepackage[T1]{fontenc}
\begin{document}

\maketitle

\begin{abstract}
The training success, training speed and generalization ability of neural networks rely crucially on the choice of random parameter initialization. It has been shown for multiple architectures that initial dynamical isometry is particularly advantageous. Known initialization schemes for residual blocks, however, miss this property and suffer from degrading separability of different inputs for increasing depth and instability without Batch Normalization or lack feature diversity. We propose a random initialization scheme, \init, that achieves perfect dynamical isometry for residual networks with ReLU activation functions even for finite depth and width. It balances the contributions of the residual and skip branches unlike other schemes, which initially bias towards the skip connections. In experiments, we demonstrate that in most cases our approach outperforms initialization schemes proposed to make Batch Normalization obsolete, including Fixup and SkipInit, and facilitates stable training. Also in combination with Batch Normalization, we find that \init often achieves the overall best result. 
\end{abstract}

\section{Introduction}

Random initialization of weights in a neural network play a crucial role in determining the final performance of the network. 
This effect becomes even more pronounced for very deep models that seem to be able to solve many complex tasks more effectively.
An important building block of many models are residual blocks \citet{he2016deep}, in which skip connections between non-consecutive layers are added to ease signal propagation \citep{balduzzi2017shattered} and allow for faster training.
ResNets, which consist of multiple residual blocks, have since become a popular center piece of many deep learning applications \citep{bello2021revisiting}.

Batch Normalization (BN) \citep{ioffe2015batch} is a key ingredient to train ResNets on large datasets.
It allows training with larger learning rates, often improves generalization, and makes the training success robust to different choices of parameter initializations. 
It has furthermore been shown to smoothen the loss landscape \citep{santurkar2018does} and to improve signal propagation \citep{skipinit}.
However, BN has also several drawbacks: It breaks the independence of samples in a minibatch and adds considerable computational costs.
Sufficiently large batch sizes to compute robust statistics can be infeasible if the input data requires a lot of memory.
Moreover, BN also prevents adversarial training \citep{wang2022removing}.
For that reason, it is still an active area of research to find alternatives to BN \citet{zhang2018fixup,brock2021high}. 
A combinations of Scaled Weight Standardization and gradient clipping has recently outperformed BN \citep{brock2021high}.
However, a random parameter initialization scheme that can achieve all the benefits of BN is still an open problem.
An initialization scheme allows deep learning systems the flexibility to drop in to existing setups without modifying pipelines.
For that reason, it is still necessary to develop initialization schemes that enable learning very deep neural network models without normalization or standardization methods. 

A direction of research pioneered by \citet{saxe2013exact,pennington2017resurrecting} has analyzed the signal propagation through randomly parameterized neural networks in the infinite width limit using random matrix theory.
They have argued that parameter initialization approaches that have the \textit{dynamical isometry} (DI) property avoid exploding or vanishing gradients, as the singular values of the input-output Jacobian are close to unity. 
DI is key to stable and fast training \citep{du2018gradient,Hu2020Provable}. 
While \citet{pennington2017resurrecting} showed that it is not possible to achieve DI in networks with ReLU activations with independent weights or orthogonal weight matrices, \citet{burkholz2019initialization,balduzzi2017shattered} derived a way to attain perfect DI even in finite ReLU networks by parameter sharing. 
This approach can also be combined \citep{constnet,balduzzi2017shattered} with orthogonal initialization schemes for convolutional layers \citep{xiao2018dynamical}. 
The main idea is to design a random initial network that represents a linear isometric map.

We transfer a similar idea to ResNets but have to overcome the additional challenge of integrating residual connections and, in particular, potentially non-trainable identity mappings while balancing skip and residual connections and creating initial feature diversity.  
We propose an initialization scheme, \init (\textbf{R}esidual dynamical \textbf{iso}me\textbf{t}ry by ini\textbf{t}ial \textbf{o}rthogonality), that achieves \textit{dynamical isometry} (DI) for ResNets \citep{he2016deep} with convolutional or fully-connected layers and ReLU activation functions exactly. 
\init achieves this for networks of finite width and finite depth and not only in expectation but exactly.
We provide theoretical and empirical evidence that highlight the advantages of our approach. 
In contrast to other initialization schemes that aim to improve signal propagation in ResNets, \init can achieve performance gains even in combination with BN.
We further demonstrate that \textit{\init} can successfully train ResNets without BN and achieve the same or better performance than \citet{zhang2018fixup, brock2021high}.

%

\subsection{Contributions}
\begin{itemize}
    \item To explain the drawbacks of most initialization schemes for residual blocks, we derive signal propagation results for finite networks without requiring mean field approximations and highlight input separability issues for large depths.
    \item We propose a solution, \init, which is an initialization scheme for residual blocks that provably achieves dynamical isometry (exactly for finite networks and not only approximately). A residual block is initialized so that it acts as an orthogonal, norm and distance preserving transform.
    \item In experiments on multiple standard benchmark datasets, we demonstrate that our approach achieves competitive results in comparison with alternatives: 
    \begin{itemize}
    \item We show that \init facilitates training ResNets without BN or any other normalization method and often outperforms existing BN free methods including Fixup, SkipInit, and NF ResNets.
    \item It outperforms standard initialization schemes for ResNets with BN on Tiny Imagenet and CIFAR100.
    \end{itemize}
\end{itemize}

\subsection{Related Work} \label{sec:bg}
\textbf{Preserving Signal Propagation}
Random initialization schemes have been designed for a multitude of neural network architectures and activation functions.
Early work has focused on the layerwise preservation of average squared signal norms \citep{glorot-bengio,he-init,HaninExplodingVanishing} and their variance \citep{startTrainingResNets}.
The mean field theory of infinitely wide networks has also integrated signal covariances into the analysis and further generated practical insights into good choices that avoid exploding or vanishing gradients and enable feature learning \citep{featureLearning} if the parameters are drawn independently \citep{poole2016exponential,raghu_expressive_2016,schoenholz_deep_2017,yang2017mean,xiao2018dynamical}. 
Indirectly, these works demand that the average eigenvalue of the signal input-output Jacobian is steered towards $1$. 
Yet, in this set-up, ReLU activation functions fail to support parameter choices that lead to good trainability of very deep networks, as outputs corresponding to different inputs become more similar for increasing depth \citep{poole2016exponential,burkholz2019initialization}.
\cite{yang2017mean} could show that ResNets can mitigate this effect and enable training deeper networks, but also cannot distinguish different inputs eventually. 

However, there are exceptions.
Balanced networks can improve \citep{li2021future} interlayer correlations and reduce the variance of the output. 
A more effective option is to remove the contribution of the residual part entirely as proposed in successful ResNet initialization schemes like Fixup \citep{zhang2018fixup} and SkipInit \citep{skipinit}.
This, however, limits significantly the initial feature diversity that is usually crucial for the training success \citep{constnet}. 
A way to address the issue for other architectures with ReLUs like fully-connected \citep{burkholz2019initialization} and convolutional \citep{balduzzi2017shattered} layers is a looks-linear weight matrix structure \citep{linInit}. 
This idea has not been transfered to residual blocks yet but has the advantage that it can be combined with orthogonal submatrices.
These matrices induce perfect dynamical isometry \citep{saxe2013exact,mishkin2015all,poole2016exponential,pennington2017resurrecting}, meaning that the eigenvalues of the initial input-output Jacobian are identical to $1$ or $-1$ and not just close to unity on average.
This property has been shown to enable the training of very deep neural networks \citep{xiao2018dynamical} and can improve their generalization ability \citep{FischerSpectrum} and training speed \cite{pennington2017resurrecting,spectralUniversality}. 
ResNets equipped with ReLUs can currently only achieve this property approximately and without a practical initialization scheme \citep{tarnowski2019dynamical} or with reduced feature diversity \citep{constnet} and potential training instabilities \citep{zhang2018fixup,skipinit}.

\paragraph{ResNet Initialization Approaches}
Fixup \citep{zhang2018fixup}, SkipInit \citep{skipinit}, and ReZero \citep{bachlechner2021rezero} have been designed to enable training without requiring BN, yet, can usually not achieve equal performance. 
Training data informed approaches have also been successful \citep{zhu2021gradinit, dauphin2019metainit} but they require computing the gradient of the input minibatches.
Yet, most methods only work well in combination with BN \citep{ioffe2015batch}, as it seems to improve ill conditioned initializations \citep{glorot-bengio,he2016deep} according to \citet{bjorck2018understandingbn}, allows training with larger learning rates \citep{santurkar2018does}, and might initially bias the residual block towards the identity enabling signal to flow through \citet{skipinit}.
The additional computational and memory costs of BN, however, have motivated research on alternatives including different normalization methods \citep{group-norm,weight-normalization,instance-normalization}.
Only recently has it been possible to outperform BN in generalization performance using scaled weight standardization and gradient clipping \citep{brock2021high, brock2021characterizing}, but this requires careful hyperparameter tuning. 
In experiments, we compare our initialization proposal \init with all three approaches: normalization free methods, BN and normalization alternatives (e.g NF ResNet). 

\section{ResNet Initialization}\label{sec:theory}
\subsection{Background and Notation}\label{sec:background}
The object of our study is a general residual network that is defined by
\begin{align}\label{eq:resdef}
    \vz^0 := \rmW^0 * \vx,\quad \vx^l = \phi(\vz^{l-1}), \quad \vz^l := \alpha_l f^l(\vx^l) + \beta_l h^l( \vx^l);  \quad \vz^{\text{out}} := \rmW^{\text{out}} P(\vx^L)
\end{align}
for $1 \leq l \leq L$.
$P(.)$ denotes an optional pooling operation like maxpool or average pool, $f(.)$ residual connections, and $h(.)$ the skip connections, which usually represent an identity mapping or a projection. 
For simplicity, we assume in our derivations and arguments that these functions are parameterized as $f^l(\vx^l) = \rmW^l_2 * \phi(\rmW^l_1 * \vx^l + \vb^l_1) + \vb^l_2$ and $h^l(\vx^l) = \rmW^l_{\text{skip}} * \vx^l + \vb^l_{\text{skip}}$ ($*$ denotes convolution), but our arguments also transfer to residual blocks in which more than one layer is skipped.
Optionally, batch normalization (BN) layers are placed before or after the nonlinear activation function $\phi(\cdot)$. 
We focus on ReLUs $\phi(x)=\max\{0,x\}$ \citep{relu}, which are among the most commonly used activation functions in practice. 
All biases $\vb^l_2 \in \R^{N_{l+1}}$, $\vb^l_1 \in \R^{N_{m_l}}$, and $\vb^l_{\text{skip}} \in \R^{N_{l}}$ are assumed to be trainable and set initially to zero. 
We ignore them in the following, since we are primarily interested in the neuron states and signal propagation at initialization.
The parameters $\alpha$ and $\beta$ balance the contribution of the skip and the residual branch, respectively.
Note that $\alpha$ is a trainable parameter, while $\beta$ is just mentioned for convenience to simplify the comparison with standard He initialization approaches \citep{he-init}. 
Both parameters could also be integrated into the weight parameters $\rmW^l_2 \in \R^{N_{l+1} \times N_{m_l} \times k^l_{2,1} \times k^l_{2,2}}$, $\rmW^l_1 \in \R^{N_{m_l} \times N_{l} \times  k^l_{1,2} \times k^l_{1,2}}$, and $\rmW^l_{\text{skip}} \in \R^{N_{l+1}\times N_{l} \times 1 \times 1}$, but they make the discussion of different initialization schemes more convenient and simplify the comparison with standard He initialization approaches \citep{he-init}.

\textbf{Residual Blocks}
Following the definition by \citet{he-init}, we distinguish two types of residual blocks, Type B and Type C (see Figure \ref{fig:res-block}), which differ in the choice of $\rmW^l_{\text{skip}}$. 
The Type C residual block is defined as $\vz^l = \alpha f^{l}(\vx^{l}) + h^l(\vx^l)$ so that shortcuts $h(.)$ are projections with a $1 \times 1$ kernel with trainable parameters. 
The type B residual block has identity skip connections $\vz^l = \alpha f^{l}(\vx^{l}) + \vx^l$.
Thus, $\rmW^l_{\text{skip}}$ represents the identity and is not trainable. 

\subsection{Signal Propagation for Normal ResNet Initialization}
Most initialization methods for ResNets draw weight entries independently at random, including FixUp and SkipInit. 
To simplify the theoretical analysis of the induced random networks and to highlight the shortcomings of the independence assumption, we assume:   
\begin{definition}[Normally Distributed ResNet Parameters]\label{def:initNormal}
All biases are initialized as zero and all weight matrix entries are independently normally distributed with\\ $w^l_{ij, 2} \sim \mathcal{N}\left(0, \sigma^2_{l,2}\right)$, $w^l_{ij, 1} \sim \mathcal{N}\left(0, \sigma^2_{l,1}\right)$, and $w^l_{ij, \text{skip} } \sim \mathcal{N}\left(0, \sigma^2_{l,\text{skip}}\right)$.  
\end{definition}
Most studies further focus on special cases of the following set of parameter choices.
\begin{definition}[Normal ResNet Initialization]\label{def:initHe}
The choice $\sigma_{l,1} = \sqrt{\frac{2}{N_{m_l} k^l_{1,1} k^l_{1,2}}}$, $\sigma_{l,2} = \sqrt{\frac{2}{N_{l+1} k^{l}_{2,1} k^{l}_{2,2}}}$, $\sigma_{l,\text{skip}} = \sqrt{\frac{2}{N_{l+1}}}$ as used in Definition~\ref{def:initNormal}  and $\alpha_l, \beta_l \geq 0$ that fulfill $\alpha^2_l + \beta^2_l = 1$.
\end{definition}
Another common choice is $\rmW_{\text{skip}} = \mathbb{I}$ instead of random entries.
If $\beta_l = 1$, sometimes also $\alpha_l \neq 0$ is still common if it accounts for the depth $L$ of the network.
In case $\alpha_l$ and $\beta_l$ are the same for each layer we drop the subscript $l$.
For instance, Fixup \citep{zhang2018fixup} and SkipInit \citep{skipinit} satisfy the above condition with $\alpha = 0$ and $\beta=1$.
\citet{skipinit} argue that BN also suppresses the residual branch effectively. 
However, in combination with He initialization \citep{he-init} it becomes more similar to $\alpha = \beta = \sqrt{0.5}$.
\citet{li2021future} study the case of free $\alpha_l$ but focus their analysis on identity mappings $\rmW^l_1 = \mathbb{I}$ and $\rmW^l_{\text{skip}} = \mathbb{I}$.

As other theoretical work, we focus our following investigations on fully-connected layers to simplify the exposition. 
Similar insights would transfer to convolutional layers but would require extra effort \citep{yang2017mean}.
The motivation for the general choice in Definition~\ref{def:initHe} is that it ensures that the average squared l2-norm of the neuron states is identical in every layer. 
This has been shown by \citet{li2021future} for the special choice $\rmW^l_1 = \mathbb{I}$ and $\rmW^l_{\text{skip}} = \mathbb{I}$, $\beta=1$ and by \citep{yang2017mean} in the mean field limit with a missing ReLU so that $\vx^l = \vz^{l-1}$.
\citep{startTrainingResNets} has also observed for $\rmW^l_{\text{skip}} = I$ and $\beta=1$ that the squared signal norm increases in $\sum_l \alpha_l$. 
For completeness, we present the most general case next and prove it in the appendix.

\begin{figure*}[h!]
    \begin{subfigure}[b]{0.3\textwidth}
        \includegraphics[height=3.5cm]{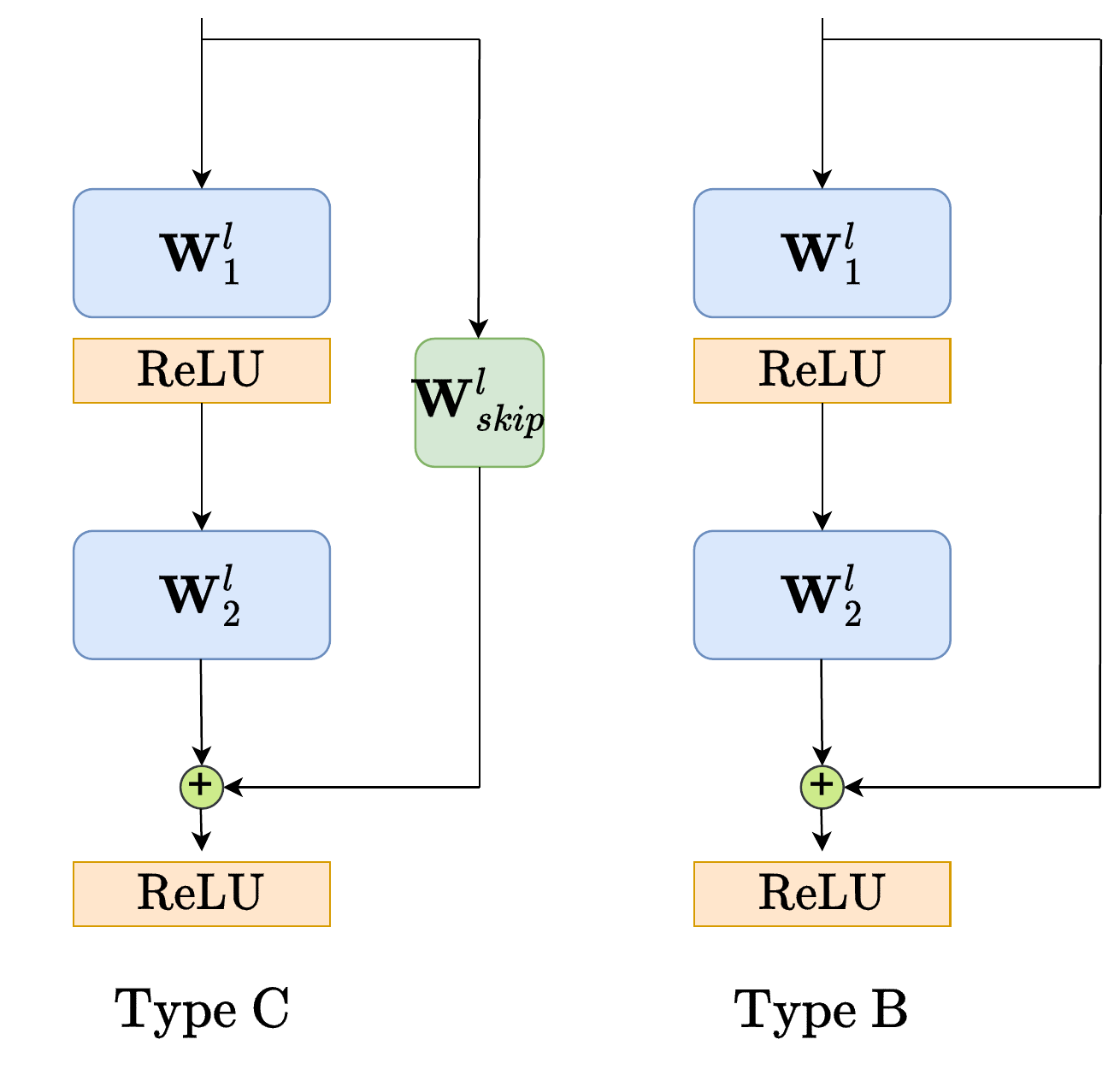}
        \caption{}
        \label{fig:res-block}
    \end{subfigure}
    \begin{subfigure}[b]{0.3\textwidth}
        \includegraphics[height=3.5cm]{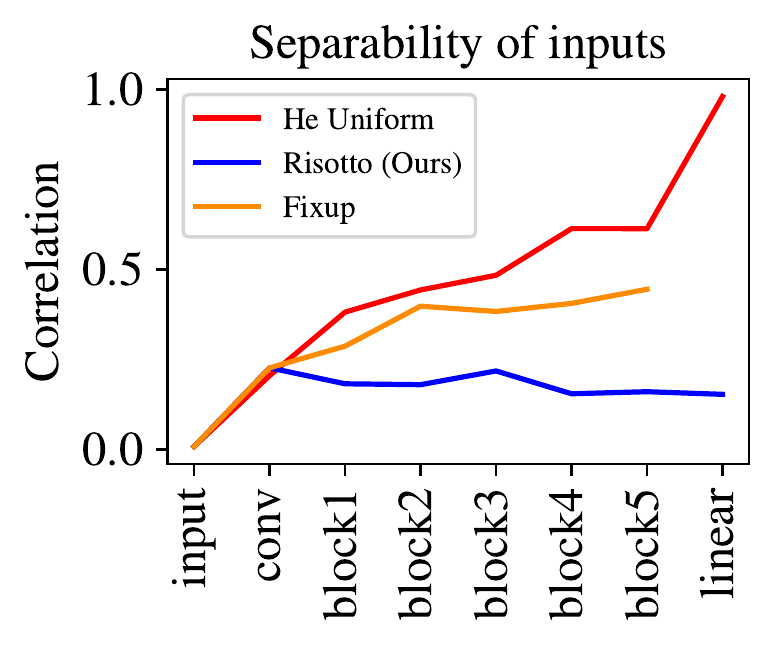}
        \caption{}
        \label{fig:correlation}
    \end{subfigure}
    \begin{subfigure}[b]{0.3\textwidth}
        \centering
        \includegraphics[height=3.5cm]{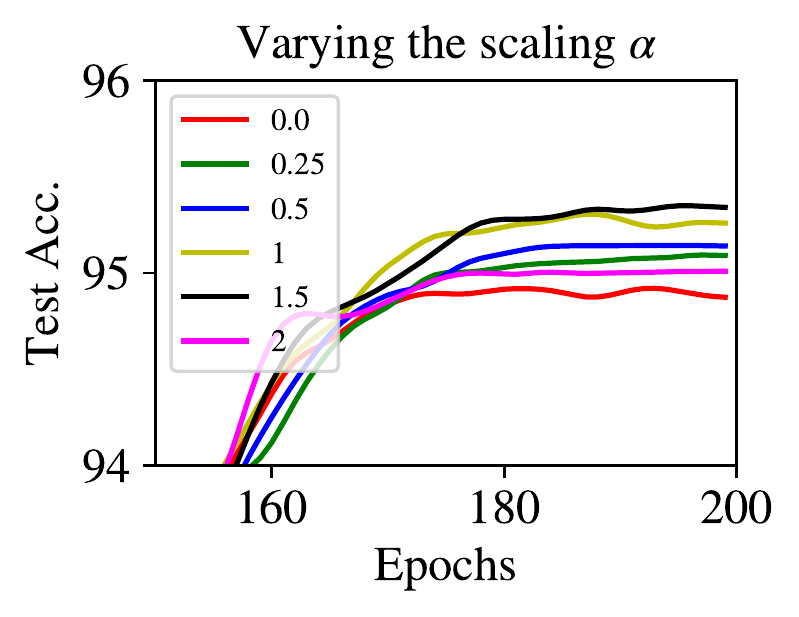}
        \caption{}
        \label{fig:alpha-sweep}
    \end{subfigure}
    \centering
        
    \caption{$(a)$The two types of considered residual blocks. 
    In Type C the skip connection is a projection with a $1\times1$ kernel while in Type B the input is directly added to the residual block via the skip connection. Both these blocks have been described by \citet{he2016deep}.
    $(b)$ The correlation between two inputs for different initializations as they pass through a residual network consisting of a convolution filter followed by $5$ residual blocks (Type C), an average pool, and a linear layer on CIFAR10. Only \init maintains constant correlations after each residual block while it increases for the other initializations with depth.
    $(c)$ Performance of \init for different values of alpha ($\alpha$) for ResNet 18 (C) on CIFAR10. 
    Note that $\alpha=0$ is equivalent to SkipInit and achieves the lowest accuracy.
    Initializing $\alpha = 1$ clearly improves performance.}

\end{figure*}
\begin{theorem}[Norm preservation]\label{thm:resnorm}
    Let a neural network consist of fully-connected residual blocks as defined by Equ.~(\ref{eq:resdef}) that start with a fully-connected layer at the beginning $\rmW^0$, which contains $N_1$ output channels.
    Assume that all biases are initialized as $0$ and that all weight matrix entries are independently normally distributed with $w^l_{ij,2} \sim \mathcal{N}\left(0, \sigma^2_{l,2}\right)$, $w^l_{ij,1} \sim \mathcal{N}\left(0, \sigma^2_{l,1}\right)$, and $w^l_{ij,\text{skip}} \sim \mathcal{N}\left(0, \sigma^2_{l,\text{skip}}\right)$. 
    Then the expected squared norm of the output after one fully-connected layer and $L$ residual blocks applied to input $\vx$ is given by 
    \begin{align*}
    \E \left( \norm{\vx^{L}}^2 \right) = \frac{N_{1}}{2} \sigma^2_0 \prod^{L-1}_{l=1} \frac{N_{l+1}}{2} \left(\alpha^2_l \sigma^2_{l,2} \sigma^2_{l,1}  \frac{N_{m_l}}{2} + \beta^2_l \sigma^2_{l,\text{skip}}  \right) \norm{\vx}^2. 
    \end{align*} 
\end{theorem}
Note that this result does not rely on any (mean field) approximations and applies also to other parameter distributions that have zero mean and are symmetric around zero. 
Inserting the parameters of Definition~\ref{def:initNormal} for fully-connected networks with $k=1$ leads to the following insight that explains why this is the preferred initialization choice.
\begin{insight}[Norm preserving initialization]\label{cor:init}
Acccording to Theorem~\ref{thm:resnorm}, the normal ResNet initialization (Definition~\ref{def:initHe}) preserves the average squared signal norm for arbitrary depth $L$. 
\end{insight}
Even though this initialization setting is able to avoid exploding or vanishing signals, it still induces considerable issues, as the analysis of the joint signal corresponding to different inputs reveals.
According to the next theorem, the signal covariance fulfills a layerwise recurrence relationship that leads to the observation that signals become more similar with increasing depth.
\begin{theorem}[Layerwise signal covariance]\label{thm:cor}
    Let a fully-connected residual block be given as defined by Eq.~(\ref{eq:resdef}) with random parameters according to Definition~\ref{def:initHe}. 
Let $\vx^{l+1}$ denote the neuron states of Layer $l+1$ for input $x$ and $\tilde{\vx}^{l+1}$ the same neurons but for input $\tilde{\vx}$. 
Then their covariance given all parameters of the previous layers is given as $\E_l \left({\langle \vx^{l+1}, \tilde{\vx}^{l+1}\rangle}\right)$
\begin{align}\label{eq:covlayerthm}
& \geq  \frac{1}{4} \frac{N_{l+1}}{2}  \left(\alpha^2 \sigma^2_{l,2} \sigma^2_{l,1} \frac{N_{m_l}}{2} +  2 \beta^2 \sigma^2_{l,\text{skip}} \right) {\langle \vx^l, \tilde{\vx}^l\rangle} + \frac{c}{4} \alpha^2 N_{l+1} \sigma^2_{l,2}  \sigma^2_{l,1} N_{m_l} \norm{\vx^l} \norm{\tilde{\vx}^l}   \\
&  + \E_{\rmW^l_1}\left(\sqrt{\left(\alpha^2 \sigma^2_{l,2} \norm{\phi(\rmW^l_1 \vx^l )}^2 + \beta^2 \sigma^2_{l,\text{skip}} \norm{\vx^l}^2\right) \left(\alpha^2 \sigma^2_{l,2} \norm{\phi(\rmW^l_1 \tilde{\vx}^l )}^2 + \beta^2 \sigma^2_{l,\text{skip}} \norm{\tilde{\vx}^l}^2\right)}\right),\nonumber
\end{align}
where the expectation $\E_l$ is taken with respect to the initial parameters $\rmW^l_2$, $\rmW^l_1$, and $\rmW^l_{\text{skip}}$ and the constant $c$ fulfills $0.24 \leq c \leq 0.25$.
\end{theorem}
Note that this statement holds even for finite networks. 
To clarify what that means for the separability of inputs, we have to compute the expectation with respect to the parameters of $\rmW_1$. 
To gain an intuition, we employ an approximation that holds for a wide intermediary network. 
\begin{insight}[Covariance of signal for different inputs increases with depth]\label{ins:cov}
Let a fully-connected ResNet with random parameters as in Definition~\ref{def:initHe} be given. It follows from Theorem~\ref{thm:cor} that the outputs corresponding to different inputs become more difficult to distinguish for increasing depth $L$. 
For simplicity, let us assume that $\norm{\vx} = \norm{\tilde{\vx}} = 1$. 
Then, in the mean field limit $N_{m_l} \rightarrow \infty$, the covariance of the signals is lower bounded by
\begin{align}\label{eq:covlayer}
\E \left({\langle \vx^L, \tilde{\vx}^L \rangle}\right)  \geq \gamma^L_1 {\langle \vx, \tilde{\vx}\rangle} + \gamma_2 \sum^{L-1}_{k=0} \gamma^k_1  = 
 \gamma^L_1 {\langle \vx, \tilde{\vx}\rangle} + \frac{\gamma_2}{1-\gamma_1} \left(1-\gamma^{L}_1\right)
\end{align}
for $\gamma_1 = \frac{1+\beta^2}{4} \leq \frac{1}{2}$ and $\gamma_2 = c (\alpha^2+2) \approx \frac{\alpha^2}{4} + \frac{1}{2}$ using $E_{l-1} \lVert \vx^l \rVert  \lVert\tilde{\vx}^l\rVert \approx 1$.
\end{insight}
Since $\gamma_1 < 1$, the contribution of the original input correlations ${\langle \vx, \tilde{\vx}\rangle}$ vanishes for increasing depth $L$.
Meanwhile, by adding constant contribution in every layer, irrespective of the input correlations, $\E \left({\langle \vx^L, \tilde{\vx}^L \rangle}\right)$ increases with $L$ and converges to the maximum value $1$ (or a slightly smaller value in case of smaller width $N_{m_l}$).
Thus, deep models essentially map every input to almost the same output vector, which makes it impossible for the initial network to distinguish different inputs and provide information for meaningful gradients.
Fig.~\ref{fig:correlation} demonstrates this trend and compares it with our initialization proposal \init, which does not suffer from this problem. 

While the general trend holds for residual as well as standard fully-connected feed forward networks ($\beta=0$), interestingly, we still note a mitigation for a strong residual branch ($\beta = 1$). 
The contribution by the input correlations decreases more slowly and the constant contribution is reduced for larger $\beta$. 
Thus, residual networks make the training of deeper models feasible, as they were designed to do \citep{he2016deep}.
This observation is in line with the findings of \citet{yang2017mean}, which were obtained by mean field approximations for a different case without ReLU after the residual block (so that $\vx^l = \vz^{l-1}$).
It also explains how ResNet initialization approaches like Fixup \citep{zhang2018fixup} and SkipInit \citep{skipinit} can be successful in training deep ResNets. They set $\alpha = 0$ and $\beta=1$. 
If $\rmW_{\text{skip}} = \mathbb{I}$, this approach even leads to dynamical isometry but trades it for very limited feature diversity \citep{constnet} and initially broken residual branch.
Figure~\ref{fig:alpha-sweep} highlights potential advantages that can be achieved by $\alpha \neq 0$ if the initialization can still maintain dynamical isometry as our proposal \init.

\subsection{Risotto: Orthogonal Initialization of ResNets for Dynamical Isometry}
Our main objective is to avoid the highlighted drawbacks of the ResNet initialization schemes that we have discussed in the last section.
We aim to not only maintain input correlations on average but exactly and ensure that the input-output Jacobian of our randomly initialized ResNet is an isometry.
All its eigenvalues equal thus $1$ or $-1$. 
In comparison with Fixup and SkipInit, we also seek to increase the feature diversity and allow for arbitrary scaling of the residual versus the skip branch. 

\textbf{Looks-linear matrix structure}
The first step in designing an orthogonal initialization for a residual block is to allow signal to propagate through a ReLU activation without loosing half of the information.
This can be achieved with the help of a looks-linear initialization \citep{linInit,burkholz2019initialization,balduzzi2017shattered}, which leverages the identity $\vx = \phi(\vx) - \phi(-\vx)$. 
Accordingly, the first layer maps the transformed input to a positive and a negative part.
A fully-connected layer is defined by $\vx^1 = \left[\hat{\vx}^1_{+}; \hat{\vx}^1_{-}\right] = \phi\left([\ermU^0; -\ermU^0] \vx\right)$ with respect to a submatrix $\ermU^0$. 
Note that the difference of both components defines a linear transformation of the input $\hat{\vx}^1_{+} -  \hat{\vx}^1_{-} = \ermU^0 \vx$.
Thus, all information about $\ermU^0 \vx$ is contained in $\vx^1$. 
The next layers continue to separate the positive and negative part of a signal.
Assuming this structure as input, the next layers $\vx^{l+1} = \phi(\rmW^{l}\vx^{l})$ proceed with the block structure  $\rmW^l = \left[\ermU^l \; - \ermU^l; \ermU^l \; - \ermU^l\right]$.
As a consequence, the activations of every layer can be separated into a positive and a negative part as $\vx^{l} = \left[\hat{\vx}^l_{+}; \hat{\vx}^l_{-}\right]$ so that $\norm{\vx^l} = \norm{\vz^{l-1}}$.
The submatrices $\ermU^l$ can be specified as in case of a linear neural network.
Thus, if they are orthogonal, they induce a neural network with the dynamical isometry property \citep{burkholz2019initialization}.
With the help of the Delta Orthogonal initialization \citep{xiao2018dynamical}, the same idea can also be transferred to convolutional layers. 
Given a matrix $\ermH \in \mathbb{R}^{N_{l+1} \times N_{l}}$, a convolutional tensor is defined as $\rmW \in \mathbb{R}^{N_{l+1} \times N_{l} \times k_1 \times k_2}$ as $\evw_{ijk'_1k'_2} = h_{ij}$ if $k'_1 = \lfloor k_1 / 2\rfloor \text{ and } k'_2 = \lfloor k_2 / 2\rfloor$ and $\evw_{ijk'_1k'_2} = 0$ otherwise.
We make frequent use of the combination of the idea behind the Delta Orthogonal initialization and the looks-linear structure.
\begin{definition}[Looks-linear structure]\label{def:lookslinear}
A tensor $\rmW \in \mathbb{R}^{N_{l+1} \times N_{l} \times k_1 \times k_2}$ is said to have looks-linear structure with respect to a submatrix $\ermU \in \mathbb{R}^{\lfloor N_{l+1}/2 \rfloor \times \lfloor N_{l}/2 \rfloor} $ if
\begin{align}
\evw_{ijk'_1k'_2} = 
    \left\{
	\begin{array}{ll}
		h_{ij}  & \text{if } k'_1 = \lfloor k_1 / 2\rfloor \text{ and } k'_2 = \lfloor k_2 / 2\rfloor,\\
		  0 & \text{otherwise},
	\end{array}
 \right.
 \;
    \ermH = \left[
    \begin{array}{ll}
            \ermU & - \ermU\\
              - \ermU &  \ermU           
            \end{array}\right]\label{eq:delta-ortho}
\end{align}
It has first layer looks-linear structure if $\ermH = \left[\ermU ; - \ermU\right]$.
\end{definition}
We impose this structure separately on the residual and skip branch but choose the corresponding submatrices wisely. 
To introduce \init, we only have to specify the corresponding submatrices for $\rmW^l_1$, $\rmW^l_2$, and $\rmW^l_{\text{skip}}$.
The main idea of Risotto is to choose them so that the initial residual block acts as a linear orthogonal map.

The Type C residual block assumes that the skip connection is a projection such that $h^l_i(\evx) = \sum_{j\in N_{l}} \ermW^l_{ij, \text{skip}} * \evx_j^l$, where $\rmW^l_{\text{skip}} \in \mathbb{R}^{N_{l+1}\times N_{l} \times 1 \times 1}$ is a trainable convolutional tensor with kernel size $1 \times 1$.
Thus, we can adapt the skip connections to compensate for the added activations of the residual branch in the following way. 
\begin{definition}[\init for Type C residual blocks]
\label{def:resc}
For a residual block of the form $\vx^{l+1} = \phi(\alpha * f^{l}(\vx^{l}) + h^l(\vx^l))$, where $f^{l}(\vx^{l}) = \rmW^l_2 * \phi( \rmW^l_1 * \vx^l)$, $ h^l(\vx^{l}) = \rmW^l_{\text{skip}} * \vx^l$, the weights $\rmW^l_1, \rmW^l_2$ and $\rmW_{\text{skip}}^l$ are initialized with looks-linear structure according to Def.~\ref{def:lookslinear} with the submatrices $\ermU^l_1$, $\ermU^l_2$ and $\ermU^l_{\text{skip}}$ respectively. The matrices $\ermU^l_1$, $\ermU^l_2$, and $\ermM^l$ be drawn independently and uniformly from all matrices with orthogonal rows or columns (depending on their dimension), while the skip submatrix is $\ermU^l_{\text{skip}} = \ermM^l - \alpha \ermU^l_2 \ermU^l_1$.
\end{definition}
The Type B residual block poses the additional challenge that we cannot adjust the skip connections initially because they are defined by the identity and not trainable.
Thus, we have to adapt the residual connections instead to compensate for the added input signal.
To be able to distinguish the positive and the negative part of the input signal after the two convolutional layers, we have to pass it through the first ReLU without transformation and thus define $\rmW^l_1$ as identity mapping. 
\begin{definition}[\init for Type B residual blocks]
\label{def:resb}
For a residual block of the form $\vx^{l+1} = \phi(\alpha * f^{l}(\vx^{l}) + \vx^l)$ where $f^{l}(\vx^{l}) = \rmW^l_2 * \phi( \rmW^l_1 * \vx^l)$, \init initializes the weight $\rmW^l_1$ as $\evw^l_{1, ijk'_1k'_2} = 1$ if $i = j, k'_1 = \lfloor k_1 / 2\rfloor, k'_2 = \lfloor k_2 / 2\rfloor$ and $\evw^l_{1, ijk'_1k'_2} = 0$ otherwise.
$\rmW^l_2$ has looks-linear structure (according to Def.~\ref{def:lookslinear}) with respect to a submatrix $\ermU^{l}_2 = \ermM^l - (1/\alpha) \mathbb{I}$, where $\ermM^l \in \mathbb{R}^{N_{l+1}/2 \times N_{l}/2}$ is a random matrix with orthogonal columns or rows, respectively.
\end{definition}

As we prove in the appendix, residual blocks initialized with \init preserve the norm of the input and cosine similarity of signals corresponding to different inputs not only on average but exactly.
This addresses the drawbacks of initialization schemes that are based on independent weight entries, as discussed in the last section.
\begin{theorem}[\textbf{\init preserves signal norm and similarity}]
\label{thm:norm}
A residual block that is initialized with \init maps input activations $\vx^l$ to output activations $\vx^{l+1}$ so that the norm $||\vx^{l+1}||^2 =  ||\vx^{l}||^2$ stays equal. 
The scalar product between activations corresponding to two inputs $\vx$ and $\tilde{\vx}$ are preserved in the sense that $\langle \hat{\vx}^{l+1}_{+} - \hat{\vx}^{l+1}_{-},  \tilde{\hat{\vx}}^{l+1}_{+} - \tilde{\hat{\vx}}^{l+1}_{-} \rangle = \langle \hat{\vx}^{l}_{+} - \hat{\vx}^{l}_{-},  \tilde{\hat{\vx}}^{l}_{+} - \tilde{\hat{\vx}}^{l}_{-} \rangle$.
\end{theorem}
The full proof is presented Appendix~\ref{app:sigprop}. 
It is straight forward, as the residual block is defined as orthogonal linear transform, which maintains distances of the sum of the separated positive and negative part of a signal. 
Like the residual block, the input-output Jacobian also is formed of orthogonal submatrices. 
It follows that \init induces perfect dynamical isometry for finite width and depth.
\begin{theorem}[\textbf{\init achieves exact dynamical isometry for residual blocks}]
\label{thm:dyn-iso}
A residual block whose weights are initialized with \init 
achieves exact dynamical isometry so that the singular values $\lambda \in \sigma(J)$ of the input-output Jacobian $J \in \mathbb{R}^{N_{l+1} \times k_{l+1} \times N_{l} \times k_{l}}$ fulfill $\lambda \in \{-1,1\}$.
\end{theorem}
The detailed proof is given in Appendix \ref{app:proof-dyniso}.
Since the weights are initialized so that the residual block acts as an orthogonal transform, also the input-output Jacobian is an isometry, which has the required spectral properties.
Drawing on the well established theory of dynamical isometry \citep{dynIsoRNNs,saxe2013exact,mishkin2015all,poole2016exponential,pennington2017resurrecting}, we therefore expect \init to enable fast and stable training of very deep ResNets, as we demonstrate next in experiments. 

\section{Experiments}\label{sec:expt}
In all our experiments, we use two kinds of ResNets consisting of residual blocks of Type B or Type C, as defined in Section~\ref{sec:background} and visualized in Fig.~\ref{fig:res-block}.
ResNet (B) contain Type B residual blocks if the input and output dimension of the block is equal and a Type B block otherwise, but a ResNet (C) has Type C residual blocks throughout.
All implementation details are described in Appendix \ref{app:exp}. 
We use a learnable scalar $\alpha$ that is initialized as $\alpha = 1$ and perform experiments on the benchmark datasets CIFAR10, CIFAR100 \citep{cifar10-dataset} and Tiny ImageNet \citep{tinyimagenet}.

Our main objective is to highlight three advantageous properties of our proposed initialization \init: 
(a) It enables stable and fast training of deep ResNets without any normalization methods and outperforms state-of-the-art schemes designed for this purpose.
(b) It can compete with the state-of-the-art normalization alternative, NF ResNets, without using any form of normalization. 
(c) It can outperform alternative initialization methods in combination with Batch Normalization (BN).

\textbf{\init without BN}
We start our empirical investigation by evaluating the performance of ResNets without any normalization layers. 
We compare our initialization scheme \init to the state-of-the-art baselines Fixup \citep{zhang2018fixup} and SkipInit \citep{skipinit}.
Both these methods have been proposed as substitutes for BN and are designed to achieve the benefits of BN by scaling down weights with depth (Fixup) and biasing signal flow towards the skip connection (Fixup and SkipInit).
Fixup has so far achieved the best performance for training ResNets without any form of normalization.
We observe that as shown in Table \ref{table:results-nobn}, \init is able to outperform both Fixup and SkipInit.
Moreover, we also observed in our experiments that Fixup and SkipInit are both susceptible to bad random seeds and can lead to many failed training runs. 
The unstable gradients at the beginning due to zero initialization of the last residual layer might be responsible for this phenomenon.
\init produces stable results. 
With ResNet (C), it also achieves the overall highest accuracy for all three datasets.
These results verify that a well balanced orthogonal initialization for residual blocks enables better training for different datasets and models of varying sizes.

\begin{table}[h!]
\centering
\begin{tabular}{c   c  c  c  c   } 
Dataset & ResNet &  \init (ours) & Fixup & SkipInit\\
\midrule 
\multirow{2}{*}{CIFAR10} & 18 (C) &  $\bm{93.71 \pm 0.11}$& $92.05 \pm 0.11$& $10.0 \pm 0.03$\\
 & 18 (B) &   $ 92.1\pm0.52 $& $ \bm{93.36 \pm 0.15}$& $ 92.57\pm0.21$ \\
\midrule 
\multirow{2}{*}{CIFAR100} & 50 (C) &$\bm{60.63\pm0.28}$ & $58.25\pm1.64$& $1\pm0$\\
 & 50 (B) &  $56.17\pm0.39$& $\bm{59.26\pm0.69}$& $42.40\pm0.67$\\
\midrule 
\multirow{2}{*}{Tiny ImageNet} & 50 (C) &  $\bm{49.51\pm0.06}$& $48.07\pm0.46$& $26.42\pm9.7$\\
 & 50 (B) &  $47.02\pm0.37$& $\bm{47.89\pm0.45}$& $31.57\pm8.6$\\

\bottomrule
\end{tabular}%
\vspace{0.1cm}
\caption{\textbf{\init as a substitute for Batch Normalization} The mean test accuracy over 3 runs and 0.95 standard confidence intervals are reported to compare \init, Fixup, and SkipInit without using BN. 
\init is able to achieve the overall best results for each of the benchmark datasets.
}
\label{table:results-nobn}
\end{table}

\textbf{\init learns faster for deep ResNets.}
 Fig.~\ref{fig:res101-nobn} for ResNet101 on CIFAR100 demonstrates \init's ability to train deep networks.
While both \init and Fixup achieve the same final performance, \init trains much faster in comparison.

\textbf{Comparison with Normalization Free ResNets}
NF ResNets \citep{brock2021characterizing} that use weight standardization have been shown to outperform BN when used in combination with adaptive gradient clipping. 
We find that \init is able to match or outperform NF ResNets with He initialization \citet{he2016deep} as shown in Fig.~\ref{fig:nfnet}.
While NF ResNets usually require careful hyperparameter tuning for gradient clipping, we observe that they train well with vanilla SGD on smaller datasets. 

\textbf{\init in combination with BN}
Despite its drawbacks, BN remains a popular method and is often implemented per default, as it often leads the best overall generalization performance for ResNets. 
Normalization free initialization schemes have been unable to compete with BN, even though Fixup has come close.
Whether the performance of batch normalized networks can still be improved is therefore still a relevant question. 
Table \ref{table:results-bn} compares \init with the two variants of He initialization for normally distributed and uniformly distributed weights \citep{he2016deep}. 
We find that \init arrives at marginally lower performance on CIFAR10 but outperforms the standard methods on both CIFAR100 and Tiny ImageNet.
These results in combination with the empirical results in Table \ref{table:results-nobn} showcase the versatility of \init, as it enables training without BN and can even improve training with BN.

\begin{table}[h!]
\centering
\begin{tabular}{c   c  c  c  c  } 
Dataset & ResNet &  \init (ours)  & He Normal & He Uniform \\
\midrule 
\multirow{2}{*}{CIFAR10} & 18 (C) &  $95.29 \pm 0.14$ & $\bm{95.38 \pm 0.15}$ & $95.32 \pm 0.07$\\
 & 18 (B) &  $94.93 \pm 0.07$ & $ \bm{94.99\pm 0.12}$ & $ 94.82\pm 0.03$\\
\midrule 
\multirow{2}{*}{CIFAR100} & 50 (C) & $73.11\pm0.70$ & $76.17\pm0.25$&$\bm{76.21\pm0.28}$ \\
 & 50 (B) & $\bm{78.45\pm0.08}$ &$77.40\pm0.25$ & $76.7\pm0.8$\\
\midrule 
\multirow{2}{*}{Tiny ImageNet} & 50 (C) & $\bm{59.47\pm0.02}$ &$50.12\pm0.98$ & $53.91\pm0.59$\\
 & 50 (B) & $\bm{58.73\pm0.29}$ & $52.21\pm2$& $55.05\pm0.6$\\

\bottomrule
\end{tabular}%
\vspace{0.1cm}
\caption{\textbf{\init with Batch Normalization} 
The mean test accuracy over 3 runs and 0.95 standard confidence intervals are reported to compare different initialization in combination with BN. 
\init outperforms the baseline methods with BN on CIFAR100 and Tiny ImageNet.} 
\label{table:results-bn}
\end{table}

\textbf{Can we reduce the number of BN layers?}
Considering the importance of BN on performance and how other methods struggle to compete with it, we empirically explore if reducing the number of BN layers in a network would still render similar benefits at reduced computational and memory costs.
We observe that in such a case, the position of the single BN layer in the network plays a crucial role.
Figure \ref{fig:cifar-bn-position} shows that \init enables training in all cases while other initializations fail if the single BN layer is not placed after the last convolutional layer to normalize all features.
BN after the last residual layer controls the norms of the logits and potentially stablizes the gradients leading to better performance.
Conversely, BN right after the first layer does not enable larger learning rates or better generalization.
In Figure \ref{fig:tinyimg-lastbn}, we show that even after optimal placement of the single BN layer, \init leads to the overall best results on Tiny ImageNet at reduced computational costs. 


\begin{figure*}[h!]
    \includegraphics[width=\textwidth]{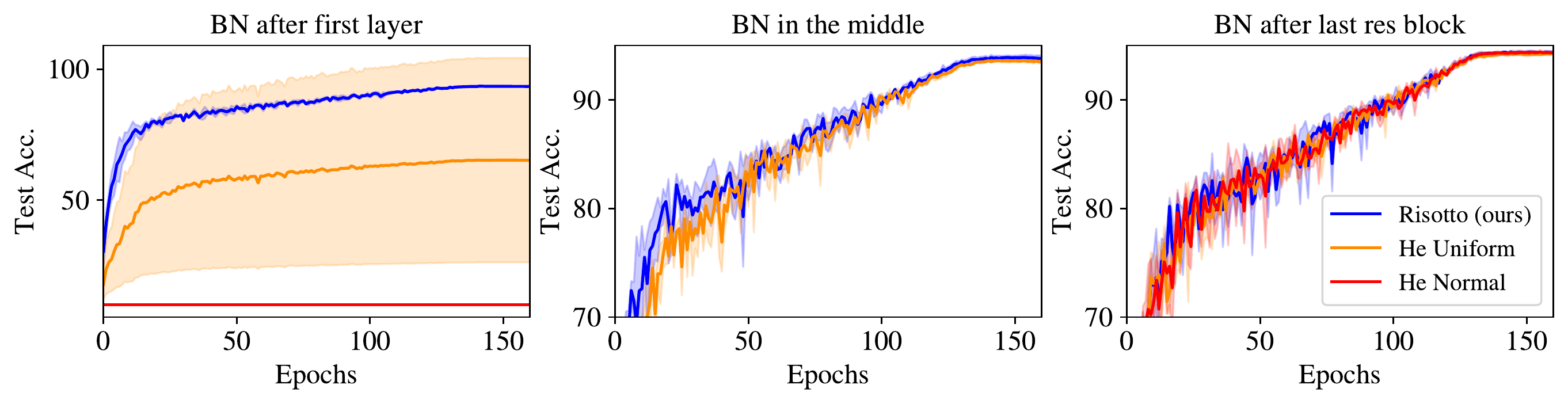}
    \caption{Comparing different positions of placing a single BN layer on a ResNet 18 (C) for CIFAR10. In each of the cases, \init allows stable training and converges to competetive accuracies, while standard methods fail in some cases.}
    \label{fig:cifar-bn-position}
\end{figure*}


\begin{figure*}[h!]
     \centering
    \begin{subfigure}[b]{0.45\textwidth}
        \centering
        \includegraphics[height=3.6cm]{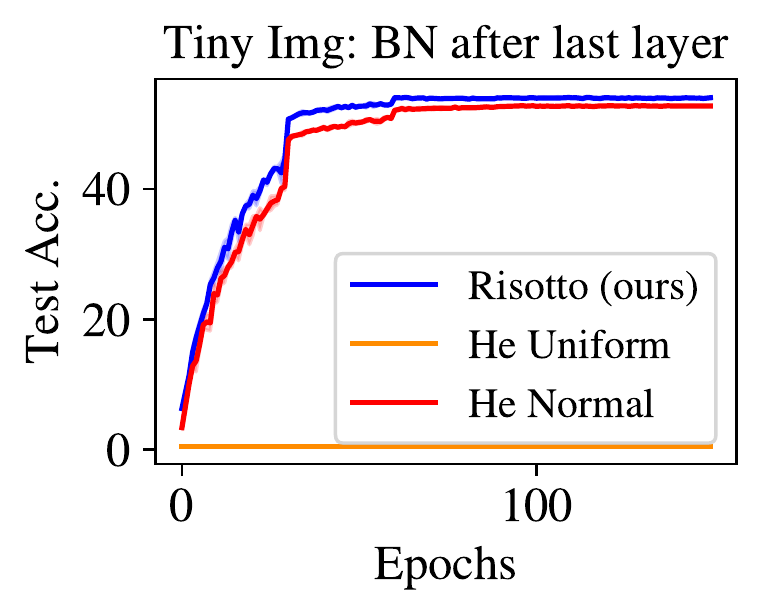}
    \end{subfigure}   
    \begin{subfigure}[b]{0.45\textwidth}
        \centering
        \includegraphics[height=3.6cm]{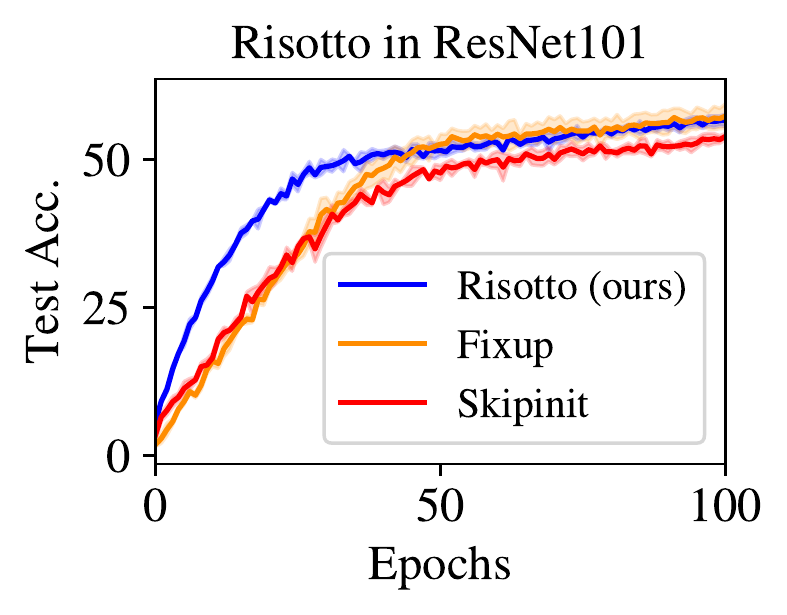}
    \end{subfigure}   
    \caption{
    $(left)$ Using only one layer of BN after the last residual block in a ResNet 50 (C) on Tiny ImageNet. 
    \init still achieves accuracy competitive to BN in every layer and performs better than Normal He initialization, while He Uniform initialization fails completely.
    $(right)$ ResNet101 (C) on CIFAR100. Training with \init achieves faster a good performance.}
\label{fig:res101-nobn}

\label{fig:tinyimg-lastbn}
\end{figure*}

\begin{figure*}[h!]
    \begin{subfigure}[b]{0.31\textwidth}
        \includegraphics[height=3.5cm]{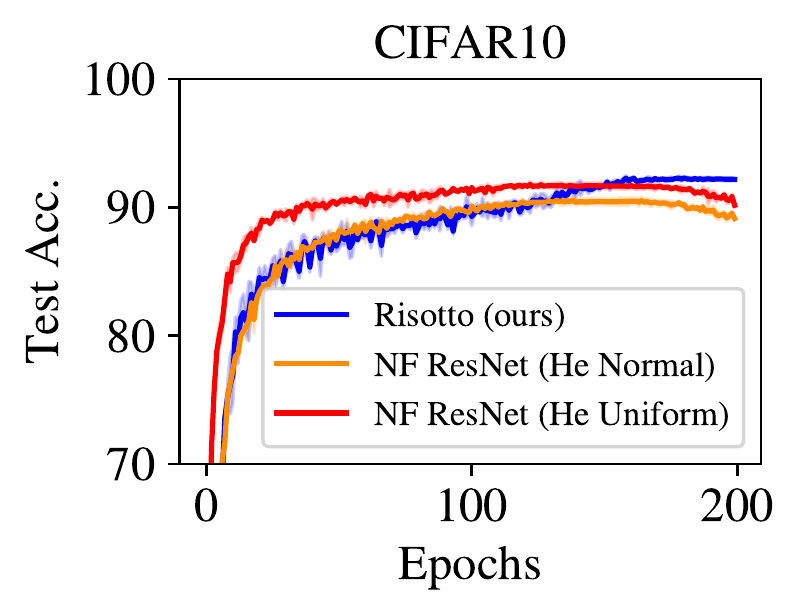}
    \end{subfigure}
    \begin{subfigure}[b]{0.31\textwidth}
        \includegraphics[height=3.5cm]{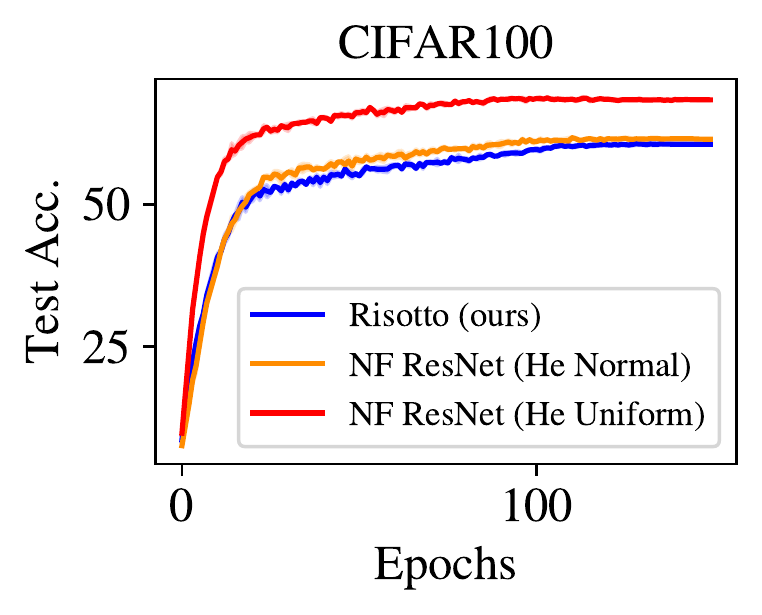}
    \end{subfigure}   
    \begin{subfigure}[b]{0.31\textwidth}
        \includegraphics[height=3.5cm]{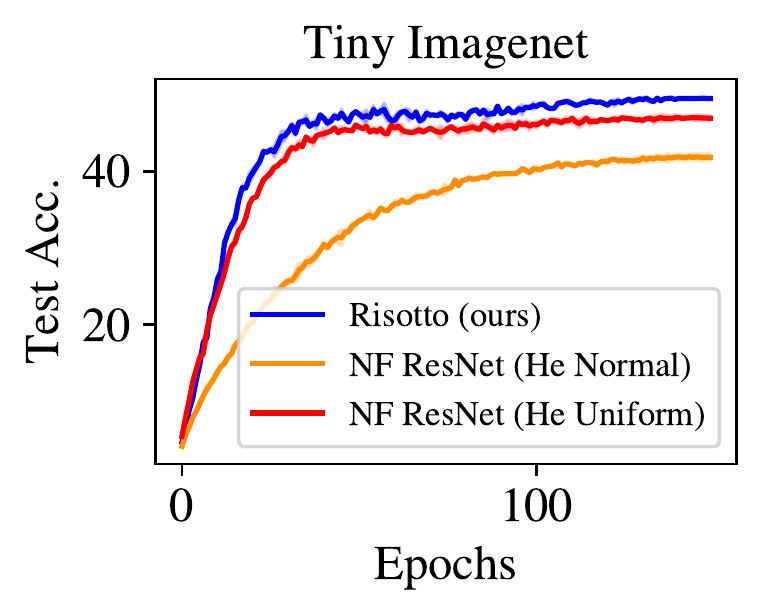}
    \end{subfigure}   
    \caption{Comparing NF ResNets to \init on all three datasets with a ResNet (C) (18 for CIFAR10 and 50 for the others). While \init performs similar to He Normal on CIFAR100, it is able to outperform both He Uniform and He Normal on Tiny ImageNet and CIFAR10.}
\label{fig:nfnet-large}
\label{fig:nfnet}
\end{figure*}

\section{Conclusions}
We have introduced a new initialization method for residual networks with ReLU activations, \init.
It enables residual networks of any depth and width to achieve exact dynamical isometry at initialization.
Furthermore, it can balance the signal contribution from the residual and skip branches instead of suppressing the residual branch initially. 
This does not only lead to higher feature diversity but also promotes stable and faster training.
In practice that is highly effective, as we demonstrate for multiple standard benchmark datasets.
We show that \init competes with and often outperforms Batch Normalization free methods and even improves the performance of ResNets that are trained with Batch Normalization. 



\bibliography{iclr2023_conference}
\bibliographystyle{iclr2023_conference}

\appendix
\section{Appendix}
\label{app:proof}

\subsection{Signal propagation}
Recall our definition of a residual block in Equation~(\ref{eq:resdef}):
\begin{align}\label{app:resdef}
    \vz^0 := \rmW^0  \vx,\quad \vx^l = \phi(\vz^{l-1}),\quad \vz^l := \alpha_l f^l(\vx^l) + \beta_l h^l( \vx^l),  \quad \vz^{out} := \rmW^{out} P(\vx^L)
\end{align}
for $1 \leq l \leq L$ with $f^l(\vx^l) = \rmW^l_2 \phi(\rmW^l_1 \vx^l)$, $h^l(\vx^l) = \rmW^l_{\text{skip}} \vx^l$, where the biases have been set to zero.
In the following theoretical derivations we focus on fully-connected networks for simplicity so that 
$\rmW^l_2 \in \R^{N_{l+1} \times N_{m_l}}$, $\rmW^l_1 \in \R^{N_{m_l} \times N_{l}}$, and $\rmW^l_{\text{skip}} \in \R^{N_{l+1}\times N_{l}}$.
The general principle could also be transferred to convolutional layers similarly to the mean field analysis by \citet{xiao2018dynamical}.

A common choice for the initialization of the parameters is defined as follows.
\begin{definition}[Normal ResNet Initialization for fullly-connected residual blocks]\label{def:initHefc}
Let a neural network consist of fully-connected residual blocks as defined by Equ.~(\ref{app:resdef}).
All biases are initialized as $0$ and all weight matrix entries are independently normally distributed with $w^l_{ij,2} \sim \mathcal{N}\left(0, \sigma^2_{l,2}\right)$, $w^l_{ij,1} \sim \mathcal{N}\left(0, \sigma^2_{l,1}\right)$, and $w^l_{ij,\text{skip}} \sim \mathcal{N}\left(0, \sigma^2_{l,\text{skip}}\right)$.
Then the Normal ResNet Initialization is defined by the choice $\sigma_{l,1} = \sqrt{\frac{2}{N_{m_l}}}$, $\sigma_{l,2} = \sqrt{\frac{2}{N_{l+1}}}$, $\sigma_{l,\text{skip}} = \sqrt{\frac{2}{N_{l+1}}}$, and $\alpha_l, \beta_l \geq 0$ that fulfill\\ $\alpha^2_l + \beta^2_l = 1$.
\end{definition}
Our objective is to analyze the distribution of the signals $\vx^{l+1}$ and $\tilde{\vx}^{l+1}$, which correspond to the random neuron state of an initial neural network that is evaluated in input $\vx^{0}$ or $\tilde{\vx}^{0}$, respectively.
More precisely, we derive the average squared signal norm and the covariance of two signals that are evaluated in different inputs.

We start with the squared signal norm and, for convenience, restate Theorem~\ref{thm:norm} before the proof.
\begin{theorem}[Theorem~\ref{thm:norm} in main manuscript]
    Let a neural network consist of residual blocks as defined by Equ.~(\ref{eq:resdef}) or  Equ.~(\ref{app:resdef}) that start with a fully-connected layer at the beginning $\rmW^0$, which contains $N_1$ output channels.
    Assume that all biases are initialized as $0$ and that all weight matrix entries are independently normally distributed with $w^l_{ij,2} \sim \mathcal{N}\left(0, \sigma^2_{l,2}\right)$, $w^l_{ij,1} \sim \mathcal{N}\left(0, \sigma^2_{l,1}\right)$, and $w^l_{ij,\text{skip}} \sim \mathcal{N}\left(0, \sigma^2_{l,\text{skip}}\right)$. 
    Then the expected squared norm of the output after one fully-connected layer and $L$ residual blocks applied to input $x$ is given by 
    \begin{align*}
    \E \left( \norm{\vx^{L}}^2 \right) = \frac{N_{1}}{2} \sigma^2_0 \prod^{L-1}_{l=1} \frac{N_{l+1}}{2} \left(\alpha^2_l \sigma^2_{l,2} \sigma^2_{l,1}  \frac{N_{m_l}}{2} + \beta^2_l \sigma^2_{l,\text{skip}}  \right) \norm{\vx}^2. 
    \end{align*} 
\end{theorem}
\begin{proof}
First, we study how the signal is transformed by a single layer before we deduce successively the final output signal.
To do so, we assume that the signal of the previous layer is given.
This means we condition the expectation on the parameters of the previous layers and thus $\vx^l$ and $\tilde{\vx}^l$.
For notational convenience, we define $\E_l(\vz) = \E \left(\vz \mid \vx^l, \tilde{\vx}^l\right)$ and skip the index $l$ in the following derivations. 
We write: $\vx = \vx^{l+1}$, $\underline{\vx} = \vx^{l}$, $\vz = \vz^{l}$, $f(\underline{\vx}) = f^l(\underline{\vx}) = \rmW^l_2 \phi(\rmW^l_1 \underline{\vx})$, $h(\underline{\vx}) = h^l(\underline{\vx}) = \rmW^l_{\text{skip}} \underline{\vx}$, $\alpha=\alpha_l$, $\beta=\beta_l$.
Given all parameters from the previous layers, we deduce 
\begin{align}\label{eq:xtoz}
\E \left( \norm{\vx}^2 \mid \underline{\vx} \right) = \sum^{N_{l+1}}_{i=1}\E \left( \left(x_i\right)^2 \right) = N_{l+1}  \E \left( \left(x_1\right)^2 \right) = \frac{N_{l+1}}{2} \E_l \left( z_1 \right)^2
\end{align}
The first equality follows from the fact that all random parameters are independent and the signal components are identically distributed. 
The third equality holds because the distribution of each signal component is symmetric around zero so that the ReLU projects half of the signal away but the contribution to the average of the squared signal is just cut in half.
We continue with 
\begin{align}
\E_l \left(z_1\right)^2 &= \E_l \left(\alpha \sum^{N_{m_l}}_{i=1} w_{2,1i} \phi\left(\sum^{N_l}_{j=1} w_{1,ij} \underline{x}_j\right)  + \beta \sum^{N_l}_{k=1} w_{\text{skip},1k}  \underline{x}_k \right)^2  \\
& = \alpha^2 \sum^{N_{m_l}}_{i=1} \E_l \left(w^2_{2,ij} \right)\E_l \left(\norm{\phi\left(\sum^{N_l}_{j=1} w_{1,ij} \underline{x}_j\right)}^2 \right) + \beta^2 \sum^{N_l}_{k=1} \E_l \left(w^2_{\text{skip},1k} \right) \underline{x}^2_k \\
& = \alpha^2 \sigma^2_{l,2} N_{m_l} \E_l \left(\norm{\phi\left(\sum^{N_l}_{j=1} w_{1,1j} \underline{x}_j\right)}^2 \right) + \beta^2 \sigma^2_{l,\text{skip}} \norm{\underline{\vx}}^2 \\
& = \alpha^2 \sigma^2_{l,2} N_{m_l} \frac{1}{2} \E_l \left(\sum^{N_l}_{j=1} w_{1,1j} \underline{x}_j\right)^2  + \beta^2 \sigma^2_{l,\text{skip}}  \norm{\underline{\vx}}^2 \label{eq:reluexp}\\ 
& = \left(\alpha^2 \sigma^2_{l,2} \sigma^2_{l,1} N_{m_l}  \frac{1}{2} + \beta^2 \sigma^2_{l,\text{skip}} N_{l+1}  \right) \norm{\underline{\vx}}^2,\label{eq:normxl}
\end{align}
as all weight entries are independent and the expectation is a linear operation.
To obtain Equation~(\ref{eq:reluexp}), we just repeated the same argument as for Equation~\ref{eq:xtoz} to take care of the ReLU. 
Afterwards, we used again the independence of the weights $w_{1,1j}$.

From repeated evaluation of Equations~(\ref{eq:xtoz}) and (\ref{eq:normxl}), we obtain
\begin{align}
\E \left( \norm{\vx^{L}}^2 \right) = \frac{N_{1}}{2} \sigma^2_0 \prod^{L-1}_{l=1} \frac{N_{l+1}}{2} \left(\alpha^2 \sigma^2_{l,2} \sigma^2_{l,1} N_{m_l}  \frac{1}{2} + \beta^2 \sigma^2_{l,\text{skip}}  \right) \norm{\vx}^2 
\end{align}
for $\vx = \vx^0$.
\end{proof}
To make sure that the signal norm neither explodes or vanishes for very deep networks, we would need to choose the weight variances so that $\frac{N_{1}}{2} \sigma^2_0 \prod^L_{l=1} \frac{N_{l+1}}{2} \left(\alpha^2 \sigma^2_{l,2} \sigma^2_{l,1} N_{m_l}  \frac{1}{2} + \beta^2 \sigma^2_{l,\text{skip}}   \right) \approx 1$.
In the common normal ResNet initialization, this is actually achieved, since $\sigma_{l,1} = \sqrt{2/N_{m_l}}$, $\sigma_{l,2} = \sigma_{l,\text{skip}} = \sqrt{2/N_{l+1}}$, $\sigma_0 = \sqrt{2/N_1}$, and $\alpha^2 + \beta^2 = 1$ even preserve the average norm in every layer.

How does this choice affect whether signals for different inputs are distinguishable?
To answer this question, we analyze the covariance of the neuron state for two different inputs. 
We begin again with analyzing the transformation of a single layer and condition on all parameters of the previous layers.
To obtain a lower bound on the covariance, the following Lemma will be helpful. 
It has been derived by \cite{burkholz2019initialization} for Theorem 5.
\begin{lemma}\label{lemma:cov}
Assume that two random variables $z_1$ and $z_2$ are jointly normally distributed as $z \sim \mathcal{N}\left(0, V\right)$ with covariance matrix $V$. Then, the covariance of the ReLU transformed variables $x_1 = \phi(z_1)$ and $x_2 = \phi(z_2)$ is
\begin{align}\label{eq:sp}
\mathbb{E}\left( x_1 x_2\right) & = \sqrt{v_{11}v_{22}} \left(g(\rho)  \rho +  \frac{\sqrt{1-\rho^2}}{2\pi} \right) \\
& 
\geq  \sqrt{v_{11}v_{22}} \left(\frac{1}{4} (\rho + 1)-\tilde{c}\right) = \frac{1}{4} v_{12} + c \sqrt{v_{11}v_{22}},
\end{align} 
where $\rho=v_{11} v_{22}/v_{12}$ and $g(\rho)$ is defined as $g(\rho) =\frac{1}{ \sqrt{2\pi}}  \int^{\infty}_0 \Phi \left(\frac{\rho}{\sqrt{1-\rho^2}} u\right) \exp\left(-\frac{1}{2} u^2\right) \; du$ for $|\rho| \neq 1$ and $g(-1) = 0$, $g(1) = 0.5$. 
The constant fulfills $0.24 \leq c \leq 0.25$.
\end{lemma}

In the following, we assume that all weight parameters are normally distributed so that we can use the above lemma. However, other parameter distributions in large networks would also lead to similar results, as the central limit theorem implies that the relevant quantities are approximately normally distributed.

We study the covariance of the signals $\vx^{l+1}=\vx$ and $\tilde{\vx}^{l+1}=\tilde{\vx}$, which correspond to the random neuron state of an initial neural network that is evaluated in input $\vx^{0}$ or $\tilde{\vx}^{0}$, respectively.

\begin{theorem}[Theorem~\ref{thm:cor} in main manuscript]
    Let a fully-connected residual block be given as defined by Equ.~(\ref{eq:resdef}) or  Equ.~(\ref{app:resdef}).
    Assume that all biases are initialized as $0$ and that all weight matrix entries are independently normally distributed with $w^l_{ij,2} \sim \mathcal{N}\left(0, \sigma^2_{l,2}\right)$, $w^l_{ij,1} \sim \mathcal{N}\left(0, \sigma^2_{l,1}\right)$, and $w^l_{ij,\text{skip}} \sim \mathcal{N}\left(0, \sigma^2_{l,\text{skip}}\right)$. 
Let $x^{l+1}$ denote the neuron states of Layer $l+1$ for input $\vx$ and $\tilde{\vx}^{l+1}$ the same neurons but for input $\tilde{\vx}$. 
Then their covariance given all parameters of the previous layers is given as $\E_l \left({\langle \vx^{l+1}, \tilde{\vx}^{l+1}\rangle}\right)$
\begin{align}
& \geq  \frac{1}{4} \frac{N_{l+1}}{2}  \left(\alpha^2_l \sigma^2_{l,2} \sigma^2_{l,1} \frac{N_{m_l}}{2} +  2 \beta^2_l \sigma^2_{l,\text{skip}} \right) {\langle \vx^l, \tilde{\vx}^l\rangle} + \frac{c}{4} \alpha^2_l N_{l+1} \sigma^2_{l,2}  \sigma^2_{l,1} N_{m_l} \norm{\vx^l} \norm{\tilde{\vx}^l}   \\
&  + \E_{\rmW^l_1}\left(\sqrt{\left(\alpha^2_l \sigma^2_{l,2} \norm{\phi(\rmW^l_1 \vx^l )}^2 + \beta^2_l \sigma^2_{l,\text{skip}} \norm{\vx^l}^2\right) \left(\alpha^2_l \sigma^2_{l,2} \norm{\phi(\rmW^l_1 \tilde{\vx}^l )}^2 + \beta^2_l \sigma^2_{l,\text{skip}} \norm{\tilde{\vx}^l}^2\right)}\right),\nonumber
\end{align}
where the expectation $\E_l$ is taken with respect to the initial parameters $\rmW^l_2$, $\tilde{\rmW}^l_1$, and $\rmW^l_\text{skip}$. 
\end{theorem}
\begin{proof}
Let us assume again that all parameters of the previous layers are given in addition to the parameters of the first residual layer $\rmW_1$ and use our notation from the proof of Theorem~\ref{thm:norm}. 

Based on similar arguments that we used for the derivation of average squared signal norm, we observe that $\vz = \vz^l$ and $\tilde{\vz} = \tilde{\vz}^l$ are jointly normally distributed. 
In particular, the components $z_i$ are identically distributed for the same input. 
The same component $z_i$ and $\tilde{z}_{i}$ for different inputs has covariance matrix $\rmV$ with entries $v_{11} = \alpha^2  \sigma^2_{l,2} \norm{\phi(\rmW_1 \underline{\vx} )}^2 + \beta^2 \sigma^2_{l,\text{skip}} \norm{\underline{\vx}}^2$, $v_{22} = \alpha^2 \sigma^2_{l,2} \norm{\phi(\rmW_1 \tilde{\underline{\vx}} )}^2 + \beta^2 \sigma^2_{l,\text{skip}} \norm{\tilde{\underline{\vx}}}^2$, and $v_{12} = \alpha^2 \sigma^2_{l,2}  \langle \phi(\rmW_1 \underline{\vx}), \phi(\rmW_1\underline{\tilde{\vx}}) \rangle + \beta^2 \sigma^2_{l,\text{skip}} \langle \underline{\vx}, \underline{\tilde{\vx}} \rangle $. 

\begin{align}
\E_l \left({\langle \vx, \tilde{\vx}\rangle}\right) & = \E_{\rmW_1} \E_l \left({\langle \vx, \tilde{\vx}\rangle} \mid \rmW_1 \right) = \sum^{N_{l+1}}_{i=1}\E_{\rmW_1} \E_l \left( \phi(z_i) \phi(\tilde{z}_i) \mid \rmW_1 \right)\\ 
& \geq  N_{l+1}  \E_{\rmW_1}\left(\frac{1}{4} v_{12} + c \sqrt{v_{11}v_{22}}\right)  = N_{l+1} \E_{\rmW_1}\left(\frac{1}{4} v_{12} + c \sqrt{v_{11}v_{22}}\right) \label{eq:corv} 
\end{align}
where we applied Lemma~\ref{lemma:cov} to obtain the inequality and used the fact that the entries of the variance matrix are identically distributed for different $(i)$ with respect to $\rmW_1$.

We can compute the first term $\E_{\rmW_1} v_{12}$ by using Lemma~\ref{lemma:cov} again, as $\rmW_1 \underline{\vx}$ and $\rmW_1 \tilde{\underline{\vx}}$ are jointly normally distributed given $\underline{\vx}$ and $\underline{\tilde{\vx}}$.
The associated covariance matrix $S$ for one component corresponding to two different inputs has entries $s_{11} = \sigma^2_{l,1} \norm{\underline{\vx}}^2$, $s_{22} = \sigma^2_{l,1} \norm{\tilde{\underline{\vx}}}^2$, and $s_{12} = \sigma^2_{l,1}  {\langle \underline{\vx}, \tilde{\underline{\vx}}\rangle}$.
Lemma~\ref{lemma:cov} gives us therefore
\begin{align}
    \E_{\rmW_1}\left(v_{12}\right) & = \E_{\rmW_1}\left(\alpha^2 \sigma^2_{l,2}  \langle \phi(\rmW_1 \underline{\vx}), \phi(\rmW_1\tilde{\underline{\vx}}) \rangle + \beta^2 \sigma^2_{l,\text{skip}} \langle \underline{\vx}, \tilde{\underline{\vx}} \rangle\right) \\
    & \geq  \alpha^2 \sigma^2_{l,2} N_{m_l} \left( \frac{1}{4} \sigma^2_{l,1}  {\langle \underline{\vx}, \tilde{\underline{\vx}}\rangle}  + c \sigma^2_{l,1} \norm{\underline{\vx}} \norm{\tilde{\underline{\vx}}} \right)  + \beta^2 \sigma^2_{l,\text{skip}} \langle \underline{\vx}, \tilde{\underline{\vx}} \rangle. \label{eq:v12}
\end{align}

Determining the second part of Equation~(\ref{eq:corv}) is more involved:
\begin{align}
&  \E_{\rmW_1}  \left(\sqrt{v_{11}v_{22}}\right) \nonumber\\
& = \E_{\rmW_1}\left(\sqrt{\left(\alpha^2 \sigma^2_{l,2} \norm{\phi(\rmW_1 \underline{\vx} )}^2 + \beta^2 \sigma^2_{l,\text{skip}} \norm{\underline{\vx}}^2\right) \left(\alpha^2 \sigma^2_{l,2} \norm{\phi(\rmW_1 \tilde{\underline{\vx}} )}^2 + \beta^2 \sigma^2_{l,\text{skip}} \norm{\tilde{\underline{\vx}}}^2\right)}\right) \label{eq:stds}\\
& =  \beta^2 \sigma^2_{l,\text{skip}} \norm{\underline{\vx}} \norm{\tilde{\underline{\vx}}} \E_{\rmW_1}\sqrt{ \left(\gamma  \norm{\phi\left(\rmW_1 \frac{\underline{\vx}}{\norm{\underline{\vx}}} \right)}^2 +1\right) \left(\gamma  \norm{\phi\left(\rmW_1 \frac{\tilde{\underline{\vx}}}{\norm{\tilde{\underline{\vx}}}} \right)}^2 +1\right)} \\
& \geq  \beta^2 \sigma^2_{l,\text{skip}} \norm{\underline{\vx}} \norm{\tilde{\underline{\vx}}} \E_{\rmW_1}\sqrt{ \left(\gamma  \norm{\phi\left(\rmW_1 \frac{\underline{\vx}}{\norm{\underline{\vx}}} \right)}^2 +1\right) \left(\gamma  \norm{\phi\left(-\rmW_1 \frac{\underline{\vx}}{\norm{\underline{\vx}}} \right)}^2 +1\right)} \label{eq:negassoc} \\
& \geq  \beta^2 \sigma^2_{l,\text{skip}} \norm{\underline{\vx}} \norm{\tilde{\underline{\vx}}}\E_{\rmW_1} \sqrt{ \left(\gamma \sum^{M}_{j=1} w_{1,j1}^2 +1\right)} \sqrt{ \left(\gamma \sum^{N_{m_l}}_{j=M+1} w_{1,j1}^2 +1\right)}\\
& =  \beta^2 \sigma^2_{l,\text{skip}}  \norm{\underline{\vx}} \norm{\tilde{\underline{\vx}}}\E_{\ermM} \E_{\ermY}  \Biggl[ \sqrt{ \left(\gamma \sigma^2_{l,1} \frac{N_{m_l}}{2} \left(\frac{2}{N_{m_l}} \sum^{M}_{j=1} y^2_{j}\right) +1\right)} \nonumber\\
 & \quad \times  \sqrt{ \left(\gamma \sigma^2_{l,1} \frac{N_{m_l}}{2} \left(\frac{2}{N_{m_l}}  \sum^{N_{m_l}}_{j=M+1} y^2_j\right) +1\right)} \Biggr]\nonumber \\
& \approx \beta^2 \sigma^2_{l,\text{skip}}  \norm{\underline{\vx}} \norm{\tilde{\underline{\vx}}} \left(\gamma \sigma^2_{l,1} \frac{N_{m_l}}{2} +1 \right) =  \norm{\underline{\vx}} \norm{\tilde{\underline{\vx}}} \left(\alpha^2 \sigma^2_{l,2} \sigma^2_{l,1} \frac{N_{m_l}}{2} + \beta^2 \sigma^2_{l,\text{skip}} \right)
\end{align}
In Equation~(\ref{eq:negassoc}), we have used the fact that $\E_{\rmW_1}  \left(\sqrt{v_{11}v_{22}}\right)$ is monotonically increasing in the covariance $s_{12}$. 
Thus, the minimum is attained for perfectly negative associations between $\vx$ and $\tilde{\vx}$ and thus $\tilde{\vx}=-\vx$.
To simplify the derivation, we further study the case $\vx = (1,0,0,0,...)^T$.
It follows that either $\phi(w_{1,j1})$ or $\phi(-w_{1,j1})$ is positive while the other one is zero.
To ease the notation, by reindexing, we can assume that the first $M$ components fulfill $\phi(w_{1,j1}) > 0$, while the remaining $N_{m_l}-M$ components fulfill the opposite.
Note that because $w_{1,j1}$ is distributed symmetrically around zero, $M \sim \text{Bin}(N_{m_l}, 0.5)$ is a binomially distributed random variable with success probability $0.5$.
Thus, $\E M = N_{m_l}/2$.
To make the dependence on $N_{m_l}$ of the different variables more obvious, we have replaced the random variables $w_{1,j1}$ that are normally distributed with standard deviation $\sigma_1$ by standard normally distributed random variables $y_j$ with standard deviation $1$. 
This makes the use of the law of large numbers in the last equation more apparent.
Note that this approximation is only accurate for large $N_{m_l} >> 1$, which is usually fulfilled in practice. 

Finally, combining Equations~(\ref{eq:corv}), (\ref{eq:v12}), and (\ref{eq:stds}), we receive
\begin{align}
& \E_l \left({\langle \vx, \tilde{\vx}\rangle}\right)  \geq  \frac{1}{16} \alpha^2 N_{l+1} \sigma^2_{l,2} \sigma^2_{l,1} N_{m_l} {\langle \underline{\vx}, \tilde{\underline{\vx}}\rangle}  + \frac{c}{4} \alpha^2 N_{l+1} \sigma^2_{l,2}  \sigma^2_{l,1} N_{m_l} \norm{\underline{\vx}} \norm{\tilde{\underline{\vx}}} + \frac{1}{4} \beta^2 N_{l+1}  \sigma^2_{l,\text{skip}} \langle \underline{\vx}, \tilde{\underline{\vx}} \rangle \nonumber \\
&  + \E_{\rmW_1}\left(\sqrt{\left(\alpha^2 \sigma^2_{l,2} \norm{\phi(\rmW_1 \underline{\vx} )}^2 + \beta^2 \sigma^2_{l,\text{skip}} \norm{\underline{\vx}}^2\right) \left(\alpha^2 \sigma^2_{l,2} \norm{\phi(\rmW_1 \tilde{\underline{\vx}} )}^2 + \beta^2 \sigma^2_{l,\text{skip}} \norm{\tilde{\underline{\vx}}}^2\right)}\right)
 \\
& \approx \frac{N_{l+1}}{2} \left(\frac{\alpha^2}{4} \frac{\sigma^2_{l,1} N_{m_l}}{2} \sigma^2_{l,2}  + \frac{\beta^2}{2} \sigma^2_{l,\text{skip}}  \right) {\langle \underline{\vx}, \tilde{\underline{\vx}}\rangle} \nonumber\\
& + c \frac{N_{l+1}}{2}  \left( \alpha^2 \frac{\sigma^2_{l,1} N_{m_l}}{2} \sigma^2_{l,2} +  \left(2 \alpha^2 \sigma^2_{l,2} \sigma^2_{l,1} \frac{N_{m_l}}{2} + 2 \beta^2 \sigma^2_{l,\text{skip}} \right) \right) \norm{\underline{\vx}} \norm{\tilde{\underline{\vx}}} 
\end{align}
\end{proof}

To understand the problem that these derivations imply, we next choose the weight parameters so that the squared norm signal is preserved from one layer to the next and, for simplicity, study the case in which $\norm{\underline{\vx}} =  \norm{\tilde{\underline{\vx}}} = 1$.
Then we have
\begin{align}
\E_l \left({\langle \vx, \tilde{\vx}\rangle}\right) & \geq  \frac{1+\beta^2}{4}  {\langle \underline{\vx}, \tilde{\underline{\vx}}\rangle}  +  c(\alpha^2  + 2)  \approx    \frac{1+\beta^2}{4}  {\langle \underline{\vx}, \tilde{\underline{\vx}}\rangle}  + \frac{\alpha^2}{4}  + \frac{1}{2}.    
\end{align}
Thus, the similarity of signals corresponding to different inputs increases always by at least a constant amount on average.
Repeating the above bound layerwise, at Layer $L$ we receive for $\gamma_1 = \frac{1+\beta^2}{4} \leq \frac{1}{2}$ and $\gamma_2 = c (\alpha^2+2)$:
\begin{align}
\E \left({\langle \vx^L, \tilde{\vx}^L \rangle}\right)  \geq \gamma^L_1 {\langle \vx, \tilde{\vx}\rangle} + \gamma_2 \sum^{L-1}_{k=0} \gamma^k_1 = 
\frac{\gamma_2}{1-\gamma_1} \left(1-\gamma^{L}_1\right).
\end{align}
According to our bound, output signals of very deep networks become more similar with increasing depth until they are almost indistinguishable. 
This phenomenon poses a great challenge for the trainability of deep resdidual neural networks with standard initialization schemes.
Note that the case without skip-connections is also covered by the choice $\alpha=1$ and $\beta=0$.
Interestingly, nonzero skip-connections ($\beta > 0$) fight the increasing signal similarity by giving more weight to the original signal similarity (increased $\gamma_1$) while decreasing the constant contribution of $\gamma_2$. This enables training of deeper models but cannot solve the general problem that increasingly deep models become worse in distinguishing different inputs initially. 
Even the best case scenario of $\alpha = 0$ and $\beta = 1$ leads eventually to forgetting of the original input association, since $\gamma_1 = 0.5 < 1$.
With $\gamma_2 \approx 0.5$, the overall signal similarity $\E \left({\langle \vx^L, \tilde{\vx}^L \rangle}\right)$ converges to $1$ for $L \rightarrow \infty$ irrespective of the input similarity.
Thus, every input signal is essentially mapped to the same vector for very deep networks, which explains the following insight.
\begin{insight}[Insight~\ref{ins:cov} in main paper]
Let a fully-connected ResNet be given whose parameters are drawn according to Definition~\ref{def:initHefc}. It follows from Theorem~\ref{thm:cor} that the outputs corresponding to different inputs become more difficult to distinguish for increasing depth $L$. 
In the mean field limit $N_{m_l} \rightarrow \infty$, the covariance of the signals is lower bounded by
\begin{align}
\E \left({\langle \vx^L, \tilde{\vx}^L \rangle}\right)  \geq \gamma^L_1 {\langle \vx, \tilde{\vx}\rangle} + \gamma_2 \sum^{L-1}_{k=0} \gamma^k_1 = \gamma^L_1 {\langle \vx, \tilde{\vx}\rangle} + 
\frac{\gamma_2}{1-\gamma_1} \left(1-\gamma^{L}_1\right)
\end{align}
for $\gamma_1 = \frac{1+\beta^2}{4} \leq \frac{1}{2}$ and $\gamma_2 = c (\alpha^2+2)$ and $E_{l-1} \lVert\vx^l\rVert  \lVert\tilde{\vx}^l\rVert \approx 1$.
\end{insight}
However, our orthogonal initialization scheme Risotto does not suffer from increasing similarity of outputs corresponding to different inputs.



\subsection{Dynamical Isometry induced by \init }
\label{app:proof-dyniso}
\begin{theorem}[Theorem \ref{thm:dyn-iso} in the main paper]
A residual block (of type B or type C) whose weights are initialized with \init 
achieves exact dynamical isometry so that the singular values $\lambda \in \sigma(J)$ of the input-output Jacobian $J \in \mathbb{R}^{N_{l+1} \times k_{l+1} \times N_{l} \times k_{l}}$ fulfill $\lambda \in \{-1,1\}$.
\end{theorem}
\textit{Proof}.
Consider a single element of the output activation of a type C residual block at layer $l$.
At initialization, an output activation component at layer $l$ is
\begin{align}
    \evx^{l+1}_{ik} &= \alpha \sum_{n \in N_{m_l}} w^l_{ink_1/2k_2/2, 2}\sum_{m\in N_l} \phi(w^l_{nmk_1/2k_2/2, 1}x_{mk}^l) + \sum_{n \in N_l} w^l_{in, \text{skip}}x^l_{nk}
\end{align}
since \init initializes only the central elements of every $2$-D filter to nonzero values and all the other values are zero which reduces the convolution to a simple summation.
For the subsequent calculations we ignore the filter dimension indices in the weights.
Now the positive part of the output $\evx^{l+1}_{ik}$ where $i \in [0, N_{l+1}/2]$ is given by
\begin{align}
    \evx^{l+1}_{ik} &= \alpha\sum_{n \in N_{m_l}} w^l_{in, 2}\sum_{j\in N_l} \phi(w^l_{nj, 1}x_{jk}^l) + \sum_{n \in N_l} w^l_{in, \text{skip}}x^l_{nk}\\
    &=  \alpha\sum_{n \in N_{m_l}/2} u^l_{in,2}\sum_{j\in N_l/2} \phi(u^l_{nj,1}x_{jk}^l) - \phi(-u^l_{nj,1}x_{jk}^l) + \sum_{n \in N_l} w^l_{in, \text{skip}}x^l_{nk}\\
    &= \alpha\sum_{n \in N_{m_l}/2}\sum_{j\in N_l/2} u^l_{in,2}u^l_{nj,1}x_{jk}^l + 
        \sum_{n \in N_l/2} \left(m^l_{in} - \alpha\sum_{n\in N_{m_l}/2} u^l_{in,2}u^l_{nj,1}\right)x^l_{nk} \\
    &= \sum_{n \in N_l/2} m^l_{in} x^l_{nk}
\end{align}
Taking the derivative of the output $x^{l+1}_{ik}$ wrt an input element $x^l_{nk}$.
\begin{align}
    \frac{\partial x^{l+1}_{ik}}{\partial x^l_{nk}} &= m^l_{in}
\end{align}
We can now obtain the input output Jacobian for $i \in [0, N_{l+1}/2]$ and $j \in [0, N_l/2]$ as 

\begin{align}
    [J]^l_{ik'_{l+1}jk'_l} = \Bigg\{
	\begin{array}{ll}
		m^l_{ij}  & \text{if } k'_{l+1} = k'_{l}\\
		  0 & \text{otherwise}
	\end{array}
\end{align}

And since $\ermM^l$ is orthogonal, the singular values of $J^l$ are one across the dimensions $i,j$.
Due to the looks linear form of the input at the weights the complete Jacobian also takes the looks linear form as
\begin{align}
    [J]^l_{ik'_{l+1}jk'_l} =
	\begin{cases}
		m^l_{ij}  & \text{if } k'_{l+1} = k'_{l}, i \in [0, N_{l+1}/2], j \in [0, N_l/2]\\
        -m^l_{ij}  & \text{if } k'_{l+1} = k'_{l}, i \in [N_{l+1}/2, N_{l+1}], j \in [0, N_l/2]\\
        -m^l_{ij}  & \text{if } k'_{l+1} = k'_{l}, i \in [0, N_{l+1}/2], j \in [N_l/2, N_l]\\
        m^l_{ij}  & \text{if } k'_{l+1} = k'_{l}, i \in [N_{l+1}/2, N_{l+1}], j \in [N_l/2, N_l]\\
		  0 & \text{otherwise}
	\end{cases}
\end{align}
The same argument follows for type B residual blocks.
Hence, since $\ermM^l$ is an orthogonal matrix, the singular values of the input output Jacobian of the residual blocks initialized with \init are exactly $\{1, -1\}$ for all depths and width and not just in the infinite width limit.

\subsection{Signal Propagation with \init}
\label{app:sigprop}

Now we closely analyze how \init transforms the input signal for residual blocks when initialized using Definitions \ref{def:resb} and \ref{def:resc}. 
\init creates effectively an orthogonal mapping that induces DI.
Note that our analysis of convolutional tensors is simplified to evaluating matrix operations since the Delta Orthogonal initialization simplifies a convolution to an effective matrix multiplication in the channel dimension of the input.
Specifically, we can track the changes in the submatrices used in the initialization of a residual block with \init and observe the output activation. 
We start with a Type B residual block evaluated at input $\vx^l = \left[\hat{\vx}^l_{+} ; \hat{\vx}^l_{-}\right]$ at Layer $l$ and set $\alpha=1$.
The residual and skip branches at Layer $l$ are then 

\begin{align*}
    f^l(\vx^l) &= \ermW^l_2 * \phi (\ermW^l_1 * \vx^l) 
            = \ermW^l_2 * \vx^l = 
            \left[
            \begin{array}{ll}
            \ermM^l - (1/\alpha) \mathbb{I} & - \ermM^l\\ 
              - \ermM^l & \ermM^l - (1/\alpha)\mathbb{I}     
            \end{array}
            \right]
             \left[
            \begin{array}{ll}
             \hat{\vx}^l_{+} \\
            \hat{\vx}^l_{-}     
            \end{array}
            \right] \\
          &= \left[
            \begin{array}{ll}
            \ermM^l\hat{\vx}^l_{+} - \ermM^l\hat{\vx}^l_{-} - (1/\alpha)\hat{\vx}^l_{+}\\
            -\ermM^l\hat{\vx}^l_{+} + \ermM^l\hat{\vx}^l_{-} - (1/\alpha)\hat{\vx}^l_{-}    
            \end{array}
            \right].
\end{align*}
Adding the skip branch to the residual branch gives
\begin{align*}
    \alpha f^l(\vx^l) + \vx^l & = \alpha \left[
            \begin{array}{ll}
            \ermM^l\hat{\vx}^l_{+} - \ermM^l\hat{\vx}^l_{-} - (1/\alpha)\hat{\vx}^l_{+}\\
            -\ermM^l\hat{\vx}^l_{+} + \ermM^l\hat{\vx}^l_{-} - (1/\alpha)\hat{\vx}^l_{-}    
            \end{array}
            \right] + 
            \left[
            \begin{array}{ll}
            \hat{\vx}^l_{+} \\
            \hat{\vx}^l_{-}
            \end{array}
            \right]
            = \alpha \left[
            \begin{array}{ll}
            \ermM^l\hat{\vx}^l_{+} - \ermM^l\hat{\vx}^l_{-}\\
            -\ermM^l\hat{\vx}^l_{+} + \ermM^l\hat{\vx}^l_{-}  
            \end{array}
            \right] \\
           \phi(\alpha f^l(\vx^l) + \vx^l)  
           & = \alpha \phi \left(\left[
            \begin{array}{ll}
            \ermM^l\hat{\vx}^l_{+} - \ermM^l\hat{\vx}^l_{-}\\
            -\ermM^l\hat{\vx}^l_{+} + \ermM^l\hat{\vx}^l_{-}  
            \end{array}
            \right]\right) 
            = \alpha \left[
            \begin{array}{ll}
            \hat{\vx}^{l+1}_{+}\\
            \hat{\vx}^{l+1}_{-}    
            \end{array}
            \right].
\end{align*}

The submatrices of the output of the residual block $\ermU_2^l\hat{\vx}^l_{+}$ and $\ermU_2^l\hat{\vx}^l_{-}$ preserve the norm of the input signal as long as we set $\alpha=1$, since $\ermU_2^l$ is an orthogonal matrix fulfilling $\norm{\ermU_2^l}^2 = 1$.
If $\alpha$ takes a value other than 1 or -1, the orthogonal matrix $\ermV^l$ has to be scaled such that $\norm{\ermU_2^l}^2 = 1/\alpha$. 
For all are experiments with type B residual blocks, however, we observe best results with $\alpha = 1$. 
Next we repeat the same computation for type C residual blocks.
Note that the type C residual block can achieve an exactly orthogonal transform of the input with \init
for any value of $\alpha$.
The signal propagates through a residual block of type C initialized with \init and input $\vx^l = \left[\hat{\vx}^l_{+} ; \hat{\vx}^l_{-}\right]$ and weights

\begin{align*}
    \ermW^l_1 = \left[
            \begin{array}{ll}
            \ermU_1^l & - \ermU_1^l\\
              - \ermU_1^l & \ermU_1^l           
            \end{array}
            \right]; 
    \ermW^l_2 = \left[
            \begin{array}{ll}
            \ermU_2^l & - \ermU_2^l\\
              - \ermU_2^l &  \ermU_2^l           
            \end{array}
            \right]
\quad    \ermW^l_{skip} = \left[\begin{array}{ll}
            \ermM^l - \alpha \ermU_2^l\ermU_1^l & - \ermM^l + \alpha\ermU_2^l\ermU_1^l\\
              - \ermM^l + \alpha\ermU_2^l\ermU_1^l &  \ermM^l - \alpha\ermU_2^l\ermU_1^l      
            \end{array}
            \right]
\end{align*}
at Layer $l$ as
\begin{align*}
   \alpha f^l(\vx^l) + h^l(\vx^l) 
    & = \alpha \ermW^l_2 \phi (\ermW^l_1 \vx) + \ermW^l_{skip} \vx \\
     & = \alpha \left[
            \begin{array}{ll}
            \ermU_2^l \phi\left(\ermU_1^l\hat{\vx}^l_{+} - \ermU_1^l\hat{\vx}^l_{+}\right) - \ermU_2^l \phi\left(-\ermU_1^l\hat{\vx}^l_{+} + \ermU_1^l\hat{\vx}^l_{+}\right)\\
            -\ermU_2^l \phi\left(\ermU_1^l\hat{\vx}^l_{+} - \ermU_1^l\hat{\vx}^l_{+}\right) + \ermU_2^l \phi\left(-\ermU_1^l\hat{\vx}^l_{+} + \ermU_1^l\hat{\vx}^l_{+}\right)\end{array}
            \right] \\
            & +
            \left[
            \begin{array}{ll}
             \phi\left((\ermM^l - \alpha\ermU_2^l\ermU_1^l)\hat{\vx}^l_{+} + (-\ermM^l + \alpha\ermU_2^l\ermU_1^l)\hat{\vx}^l_{-}\right)\\
             \phi\left(((-\ermM^l + \alpha\ermU_2^l\ermU_1^l))^l\hat{\vx}^l_{+} - (\ermM^l - \alpha\ermU_2^l\ermU_1^l)\hat{\vx}^l_{-}\right)
            \end{array}
            \right]
            \\
     & = \left[
            \begin{array}{ll}
            \ermM^l \hat{\vx}^l_{+} - \ermM^l \hat{\vx}^l_{-}\\
            -\ermM^l \hat{\vx}^l_{+} + \ermM^l \hat{\vx}^l_{-}
            \end{array}
            \right] \\
    \vx^{l+1} &= 
    \phi(\alpha f^l(\vx^l) + h^l(\vx^l))   = \left[
            \begin{array}{ll}
            \hat{\vx}^{l+1}_{+}\\
            \hat{\vx}^{l+1}_{-}
            \end{array}
            \right].
\end{align*}
We conclude that for type C residual blocks the output is of looks-linear form and has the same norm as the input because of the orthogonal submatrices and the looks-linear structure.

We now use the above formulations to prove that \init preserves the squared signal norm and similarities between inputs for residual blocks, the key property that allows stable training.


\begin{theorem}[\textbf{\init preserves signal norm and similarity}]
A residual block that is initialized with \init maps input activations $\vx^l$ to output activations $\vx^{l+1}$ so that the norm $||\vx^{l+1}||^2 =  ||\vx^{l}||^2$ stays equal. 
The scalar product between activations corresponding to two inputs $\vx$ and $\tilde{\vx}$ are preserved in the sense that $\langle \hat{\vx}^{l+1}_{+} - \hat{\vx}^{l+1}_{-},  \tilde{\hat{\vx}}^{l+1}_{+} - \tilde{\hat{\vx}}^{l+1}_{-} \rangle = \langle \hat{\vx}^{l}_{+} - \hat{\vx}^{l}_{-},  \tilde{\hat{\vx}}^{l}_{+} - \tilde{\hat{\vx}}^{l}_{-} \rangle$.
\end{theorem}
\label{app:proof-norm}


\begin{proof}
    We first prove that the squared signal norms are preserved for both types of residual blocks followed by the same for similarity between inputs.
Consider a type C residual block. 
The preactivation of the previous layer of looks linear form $\vz^{l-1} = \left[ \tilde{\vz}^{l-1}; -\tilde{\vz}^{l-1}\right]$.
The preactivation of the current layer is then given as the signal passes through the residual block as
\begin{align}
    \vz^l &= \alpha f^l(\phi(\vz^{l-1})) + h^l(\phi(\vz^{l-1})) \\
    & = \alpha \left[
            \begin{array}{ll}
            \ermU_2^l \phi\left(\ermU_1^l\hat{\vx}^l_{+} - \ermU_1^l\hat{\vx}^l_{+}\right) - \ermU_2^l \phi\left(-\ermU_1^l\hat{\vx}^l_{+} + \ermU_1^l\hat{\vx}^l_{+}\right)\\
            -\ermU_2^l \phi\left(\ermU_1^l\hat{\vx}^l_{+} - \ermU_1^l\hat{\vx}^l_{+}\right) + \ermU_2^l \phi\left(-\ermU_1^l\hat{\vx}^l_{+} + \ermU_1^l\hat{\vx}^l_{+}\right)\end{array}
            \right] \\
            & +
            \left[
            \begin{array}{ll}
             \phi\left((\ermM^l - \alpha\ermU_2^l\ermU_1^l)\hat{\vx}^l_{+} + (-\ermM^l + \alpha\ermU_2^l\ermU_1^l)\hat{\vx}^l_{-}\right)\\
             \phi\left(((-\ermM^l + \alpha\ermU_2^l\ermU_1^l))^l\hat{\vx}^l_{+} - (\ermM^l - \alpha\ermU_2^l\ermU_1^l)\hat{\vx}^l_{-}\right)
            \end{array}
            \right]\nonumber
            \\
     & = \left[
            \begin{array}{ll}
            \ermM^l \hat{\vx}^l_{+} - \ermM^l \hat{\vx}^l_{-}\\
            -\ermM^l \hat{\vx}^l_{+} + \ermM^l \hat{\vx}^l_{-}
            \end{array}
            \right] \\
    & = \left[
            \begin{array}{ll}
            \ermM^l \phi(\hat{\vz}^{l-1}) - \ermM^l \phi(-\hat{\vz}^{l-1})\\
            -\ermM^l \phi(\hat{\vz}^{l-1}) + \ermM^l \phi(-\hat{\vz}^{l-1})
            \end{array}
            \right]
\end{align}
Then the squared norm is given by 
\begin{align}
    \norm{\vz^l}^2 &= \norm{\ermM^l \phi(\hat{\vz}^{l-1}) - \ermM^l \phi(-\hat{\vz}^{l-1})}^2 + \norm{-\ermM^l \phi(\hat{\vz}^{l-1}) + \ermM^l \phi(-\hat{\vz}^{l-1})}^2 \\
    &= 2\norm{\ermM^l\phi(\hat{\vz}^{l-1})}^2 + 2\norm{\ermM^l\phi(-\hat{\vz}^{l-1})}^2 \\
    &= 2\left(\norm{\phi(\hat{\vz}^{l-1})}^2 + \norm{\phi(-\hat{\vz}^{l-1})}^2 \right)\\
    &= 2\norm{\hat{\vz}^{l-1}}^2 \\
    &= \norm{\vz^{l-1}}^2 \label{eq:z-norm}.
\end{align}
We have used the fact the $\ermM^l$ is an orthogonal matrix such that $\norm{\ermM^l x} = \norm{x}$ and the identity for ReLU activations by which $z = \phi(z) - \phi(-z)$.
Since $\vx^{l+1} = \phi(\vz^l)$, we have 
\begin{align}
\vx^{l+1} &= \phi(\vz^l) = \left[
            \begin{array}{ll}
            \phi(\hat{\vz}^l)\\
            \phi(-\hat{\vz}^l)
            \end{array}
            \right].
\end{align}
Taking the squared norm of the output then gives
\begin{align}
    \norm{\vx^{l+1}}^2 &= \norm{\phi(\hat{\vz}^l)}^2 + \norm{\phi(-\hat{\vz}^l)}^2 = \norm{\hat{\vz}^l}^2. \label{eq:x-norm}
\end{align}
The second equality results from squaring the ReLU identity as the term $\langle \phi(\hat{\vz}^l), \phi(-\hat{\vz}^l)\rangle = 0$.
Combining Eqns.~\ref{eq:z-norm} and \ref{eq:x-norm}, we obtain
\begin{align}
    \norm{\vx^{l+1}}^2 = \norm{\vx^{l}}^2.
\end{align}
\end{proof}

In fact, the norm preservation is a special case of the preservation of the scalar product, which we prove next.

Consider two independent inputs $\vx$ and $\tilde{\vx}$, their corresponding input activations at Layer $l$ are $\vx^{l}$ and $\tilde{\vx}^{l}$.
We show that as a result of \init the correlation between the preactivations is preserved which in turn means that the similarity between activations of the looks linear form is preserved as 
 $\langle \hat{\vx}_+^{l+1} - \hat{\vx}_-^{l+1}, \tilde{\hat{\vx}}_+^{l+1} - \tilde{\hat{\vx}}_-^{l+1}\rangle$ 

\begin{align}
    &= \left[
            \phi(\hat{\vz}^l) -
            \phi(-\hat{\vz}^l)
            \right]^T
            \left[
            \phi(\tilde{\hat{\vz}}^l) -
            \phi(-\tilde{\hat{\vz}}^l)
            \right]\\
    &= \left[\phi\left(\ermM^l \phi(\hat{\vz}^{l-1}) - \ermM^l \phi(-\hat{\vz}^{l-1})\right) - \phi\left(-\ermM^l \phi(\hat{\vz}^{l-1}) + \ermM^l \phi(-\hat{\vz}^{l-1})\right)\right]^T \nonumber\\
    &\times\left[\phi\left(\ermM^l \phi(\tilde{\hat{\vz}}^{l-1}) - \ermM^l \phi(-\tilde{\hat{\vz}}^{l-1})\right) - \phi\left(-\ermM^l \phi(\tilde{\hat{\vz}}^{l-1}) + \ermM^l \phi(-\tilde{\hat{\vz}}^{l-1})\right)\right]\\
    &= \left[\left(\phi(\ermM^l \phi(\hat{\vz}^{l-1})) - \phi(-\ermM^l \phi(\hat{\vz}^{l-1}))\right) - \left(\phi(\ermM^l \phi(-\hat{\vz}^{l-1})) - \phi(-\ermM^l \phi(-\hat{\vz}^{l-1}))\right)\right]^T\nonumber\\
    &\times \left[\left(\phi(\ermM^l \phi(\tilde{\hat{\vz}}^{l-1})) - \phi(-\ermM^l \phi(\tilde{\hat{\vz}}^{l-1}))\right) - \left(\phi(\ermM^l \phi(-\tilde{\hat{\vz}}^{l-1})) - \phi(-\ermM^l \phi(-\tilde{\hat{\vz}}^{l-1}))\right)\right]\\
    &= \left[\ermM^l\left(\phi(\hat{\vz}^{l-1}) - \phi(-\hat{\vz}^{l-1})\right)\right]^T\left[\ermM^l\left(\phi(\tilde{\hat{\vz}}^{l-1}) - \phi(-\tilde{\hat{\vz}}^{l-1})\right)\right]\\
    &= \left[
            \phi(\hat{\vz}^{l-1}) -
            \phi(-\hat{\vz}^{l-1})
            \right]^T
            \left[
            \phi(\tilde{\hat{\vz}}^{l-1}) -
            \phi(-\tilde{\hat{\vz}}^{l-1})
            \right]\\
    &= \langle \hat{\vx}_+^{l} - \hat{\vx}_-^{l}, \tilde{\hat{\vx}}_+^{l} - \tilde{\hat{\vx}}_-^{l}\rangle
\end{align}

The above derivations follow from the looks -linear structure of the weights and the input as well as the orthogonality of matrix $\ermM^l$. 
The same proof strategy for both norm preservation and similarity can be followed for type B residual blocks using the signal propagation \ref{app:sigprop}.
This concludes the proof.


\subsection{Experimental setup and details}
\label{app:exp}

In all our experiments we use Stochastic Gradient Descent (SGD) with momentum $0.9$ and weight $0.0005$.
We use $4$ NVIDIA A100 GPUs to train all our models.
All experiments are repeated for $3$ runs and we report the mean and $0.95$ confidence intervals.
In experiments with ResNet101, we used a learning rate of $0.005$ for all initialization schemes including ours.

\paragraph{Placing a single BN layer}
In order to identify the best position to place a single BN layer in a ResNet, we experiments with $3$ different positions. 
$(i)$ First layer: The BN was placed right after the first convolution layer before the residual blocks.
$(ii)$ BN in the middle: The BN layer was placed after half of the residual blocks in the network.
$(iii)$ BN after last res block: In this case BN was placed before the pooling operation right after the last residual block.

\paragraph{Correlation comparison in Figure \ref{fig:correlation}} 
In order to compare the correlation between inputs for different initialization schemes we use a vanilla Residual Network with five residual blocks, each consisting of the same number of channels ($32$) and a kernel size of $(3,3)$ followed by n average pooling and a linear layer.
The figure shows the correlation between two random samples of CIFAR10 averaged over $50$ runs.

\paragraph{Tiny Imagenet} Note that we use the validation set provided by the creators of Tiny Imagenet \citep{tinyimagenet} as a test set to measure the generalization performance of our trained models.

\begin{table}[h!]
\centering
\begin{tabular}{c   c  c  c  c  c c c  } 
 &&  \multicolumn{3}{c}{without BN} & \multicolumn{3}{c}{with BN} \\

\midrule
Type & Param &  \init (ours) & Fixup & Skipinit & \init (ours) & He Uniform & He Normal\\
\midrule 
\multirow{4}{*}{C} & LR & $0.1$& $0.01$& $0.01$ & $0.1$& $0.1$& $0.1$\\
& BS & $256$& $256$& $256$ & $256$ & $256$ & $256$\\
  & Schedule & cosine& cosine& cosine &cosine& cosine& cosine\\
  & Epochs & $150$& $150$& $150$ & $150$ & $150$ &$150$\\
\midrule 
\multirow{4}{*}{B} & LR & $0.1$& $0.05$& $0.05$ & $0.1$& $0.1$& $0.1$\\
  & BS & $256$& $256$& $256$ & $256$ & $256$ & $256$\\
  & Schedule & cosine& cosine& cosine &cosine& cosine& cosine\\  
  & Epochs & $150$& $150$& $150$ & $150$ & $150$ &$150$ \\
\bottomrule
\end{tabular}%
\vspace{0.1cm}
\caption{Implementation details for ResNet18 on CIFAR10}
\label{table:hyperparams-cifar10}
\end{table}

\begin{table}[h!]
\centering
\begin{tabular}{c   c  c  c  c  c c c  } 
 &&  \multicolumn{3}{c}{without BN} & \multicolumn{3}{c}{with BN} \\

\midrule
Type & Param &  \init (ours) & Fixup & Skipinit & \init (ours) & He Uniform & He Normal\\
\midrule 
\multirow{4}{*}{C} & LR & $0.01$& $0.01$& $0.001^*$ & $0.1$& $0.1$& $0.1$\\
& BS & $256$& $256$& $256$ & $256$ & $256$ & $256$\\
  & Schedule & cosine& cosine& cosine &cosine& cosine& cosine\\
  & Epochs & $150$& $150$& $150$ & $150$ & $150$ &$150$\\
\midrule 
\multirow{4}{*}{B} & LR & $0.01$& $0.01$& $0.01$ & $0.1$& $0.1$& $0.1$\\
  & BS & $256$& $256$& $256$ & $256$ & $256$ & $256$\\
  & Schedule & cosine& cosine& cosine &cosine& cosine& cosine\\  
  & Epochs & $150$& $150$& $150$ & $150$ & $150$ &$150$ \\
\bottomrule
\end{tabular}%
\vspace{0.1cm}
\caption{Implementation details for ResNet50 on CIFAR100. $*$ denotes that SkipInit failed to train even at a very low learning rate for multiple runs as reported in Table~\ref{table:results-nobn} in the main paper.}
\label{table:hyperparams-cifar100}
\end{table}

\begin{table}[h!]
\centering
\begin{tabular}{c   c  c  c  c  c c c  } 
 &&  \multicolumn{3}{c}{without BN} & \multicolumn{3}{c}{with BN} \\

\midrule
Type & Param &  \init (ours) & Fixup & Skipinit & \init (ours) & He Uniform & He Normal\\
\midrule 
\multirow{4}{*}{C} & LR & $0.01$& $0.01$& $0.001$ & $0.01$& $0.01$& $0.01$\\
& BS & $256$& $256$& $256$ & $256$ & $256$ & $256$\\
  & Schedule & cosine& cosine& cosine & step ($30, 0.1$)& step ($30, 0.1$) & step ($30, 0.1$)\\
  & Epochs & $150$& $150$& $150$ & $150$ & $150$ &$150$\\
\midrule 
\multirow{4}{*}{B} & LR & $0.01$& $0.01$& $0.01$ & $0.01$& $0.01$& $0.01$\\
  & BS & $256$& $256$& $256$ & $256$ & $256$ & $256$\\
  & Schedule & cosine& cosine& cosine &cosine& cosine& cosine\\  
  & Epochs & $150$& $150$& $150$ & $150$ & $150$ &$150$ \\
\bottomrule
\end{tabular}%
\vspace{0.1cm}
\caption{Implementation details for ResNet50 on Tiny Imagenet. Arguments for step denote that learning rate was reduced by a factor of $0.1$ every $30$ epochs.}
\label{table:hyperparams-ti}
\end{table}

\subsection{Additional experiments for BN layer placement on CIFAR100}
\label{app:cifar100-bnpos}
\begin{figure*}[h!]
    \includegraphics[width=\textwidth]{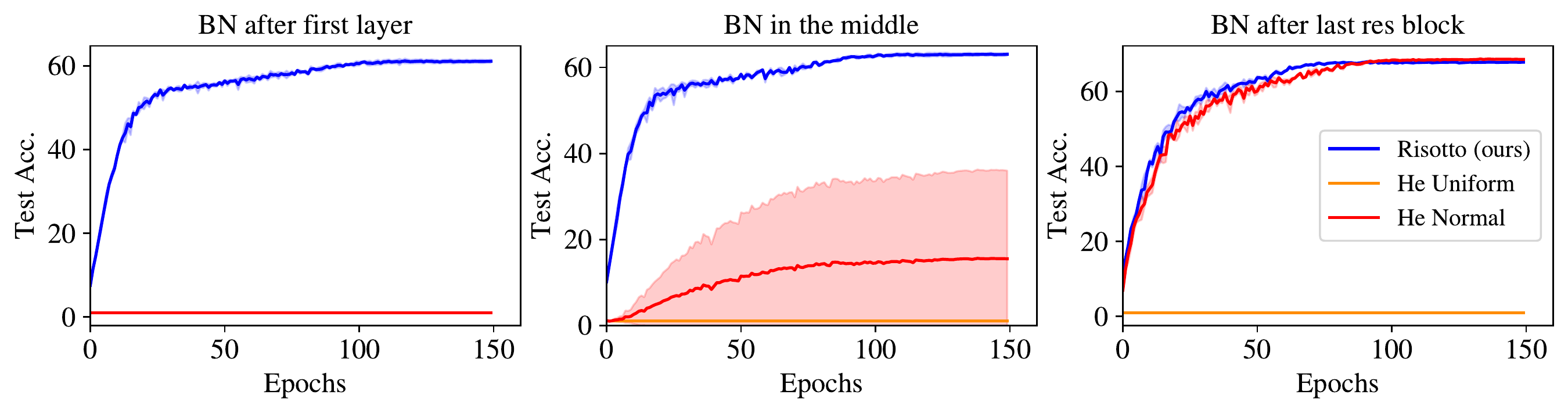}
    \caption{Comparing different positions of placing a single BN layer on CIFAR100. In each case, \init allows stable training and converges to competitive test accuracies, while standard methods fail in some cases. Standard methods are more unstable in this case compared to CIFAR10 and only He Normal is competetive when BN is placed in the last layer.}
    \label{fig:cifar100-bn-position}
\end{figure*}
In addition to our experiments on CIFAR10 (see Figure \ref{fig:cifar-bn-position} in main paper), we also report results for CIFAR100 with a ResNet50 (C) to identify the best position to place a single BN layer.
We again observe a similar trend. Even when the single BN layer is placed optimally before the last layer, \init is able to achieve the best generalization performance compared to the other methods.

\end{document}